\documentclass{article}

\usepackage{amssymb,amsthm,amsmath}
\usepackage[left=2.5cm, right=2.5cm, top=3cm, bottom=3cm]{geometry}
\usepackage{multirow}
\usepackage{booktabs}
\usepackage{mathtools}
\usepackage{lmodern}
\usepackage{enumitem}
\usepackage{algorithmic} 
\usepackage[boxed]{algorithm2e} 
\usepackage{tikz}
\usepackage{multirow}
\usepackage{subcaption}
\usepackage{xcolor}
\usetikzlibrary{calc,arrows,positioning,spy}
\tikzstyle{text_fill}=[rectangle,draw,fill=white,inner sep=.2ex]

\usepackage{pgfplots}
\pgfplotsset{compat=1.15}

\usepackage[hypertexnames=false]{hyperref}
\usepackage{cleveref}

\newcommand{\dx}{\mathrm{d}}
\newcommand{\R}{\mathbb{R}}
\newcommand{\N}{\mathbb{N}}
\newcommand{\E}{\mathbb{E}}
\newcommand{\Prob}{\mathbb{P}}
\newcommand{\Wasserstein}{\mathcal{W}}
\newcommand{\WassersteinDiscrete}{\mathbf{W}}
\newcommand{\autoencoder}{\mathbf{AE}}
\newcommand{\DCT}{\mathbf{DCT}}
\newcommand{\ID}{\mathbf{ID}}
\newcommand{\patch}{\mathbf{p}}
\newcommand{\F}{\mathbf{F}}
\newcommand{\Patch}{\mathbf{P}}

\newcommand{\Tmax}{T_\mathrm{max}}
\newcommand{\Id}{\mathrm{Id}}
\newcommand{\loss}{\ell}
\newcommand{\init}{\mathrm{init}}

\renewcommand{\div}{{\operatorname{div}}}
\newcommand{\supervised}{\mathrm{Su}}
\newcommand{\unsupervised}{\mathrm{Un}}
\newcommand{\explicit}{\mathrm{expl}}
\newcommand{\implicit}{\mathrm{impl}}
\newcommand{\Newton}{\mathrm{EN}}
\newcommand{\training}{\mathcal{P}}
\newcommand{\PSNR}{\mathrm{PSNR}}

\DeclareMathOperator*{\argmin}{argmin}
\DeclareMathOperator*{\diag}{diag}
\DeclareMathOperator{\tr}{tr}

\theoremstyle{plain}
\newtheorem{theorem}{Theorem}[section]

\theoremstyle{remark}
\newtheorem{remark}[theorem]{Remark}

\theoremstyle{definition}
\newtheorem{definition}[theorem]{Definition}

\begin{document}

\title{Shared Prior Learning of Energy-Based Models for Image Reconstruction}

\author{
Thomas Pinetz\thanks{Graz University of Technology and Austrian Institute of Technology} \and 
Erich Kobler\thanks{Graz University of Technology} \and
Thomas Pock\footnotemark[2] \and
Alexander Effland\thanks{Graz University of Technology and Silicon Austria Labs (TU Graz SAL DES Lab)}
}

\maketitle

\begin{abstract}
We propose a novel learning-based framework for image reconstruction particularly designed for training without ground truth data, which has three major building blocks: energy-based learning, a patch-based Wasserstein loss functional, and shared prior learning.
In energy-based learning, the parameters of an energy functional composed of a learned data fidelity term and a data-driven regularizer are computed in a mean-field optimal control problem.
In the absence of ground truth data, we change the loss functional to a patch-based Wasserstein functional, in which local statistics of the output images are compared to uncorrupted reference patches.
Finally, in shared prior learning, both aforementioned optimal control problems are optimized simultaneously with shared learned parameters of the regularizer to further enhance unsupervised image reconstruction.
We derive several time discretization schemes of the gradient flow and verify their consistency in terms of Mosco convergence.
In numerous numerical experiments, we demonstrate that the proposed method generates state-of-the-art results for various image reconstruction applications--even if no ground truth images are available for training.
\end{abstract}

\section{Introduction}
In this paper, we are concerned with \emph{inverse problems} in image reconstruction, in which the unknown \emph{ground truth image}~$y$ should be recovered from an \emph{observation}~$z$.
The underlying forward model is of the form
\begin{equation}
z=Z(Ay,\zeta),
\label{eq:IP}
\end{equation}
where $\zeta$ is a random variable modeling \emph{(i.i.d.)~noise}.
We assume that the \emph{task-dependent linear operator} $A$ and the \emph{observation-generating function} $Z$ are fixed.
In the case of additive Gaussian noise, for instance, $A$ is the identity operator, $\zeta\sim\mathcal{N}(0,\sigma^2\Id)$ and $Z(y,\zeta)=y+\zeta$.
Numerous problems in imaging such as denoising, deblurring, single image super-resolution, and demosaicing can be cast exactly into this form.

Inverse problems are often solved using variational methods, which allow for a Bayesian interpretation.
Here, Bayes' theorem implies that the posterior probability~$\Prob(x\vert z)$ of the reconstructed image~$x$ given the observation~$z$ is proportional to the product of the data likelihood~$\Prob(z\vert x)$ and the prior~$\Prob(x)$.
From a variational perspective, the data fidelity term~$\mathcal{D}(Ax,z)$ is identified with $-\log\Prob(z\vert x)$ and the regularizer~$\mathcal{R}(x)$ with $-\log\Prob(x)$.
Hence, maximizing the posterior probability in a logarithmic domain defines the maximum a posteriori (MAP) estimator, which is equivalent to minimizing the \emph{energy functional}~$\mathcal{E}$ given by
\[
\mathcal{E}(x,z)=\mathcal{D}(Ax,z)+\mathcal{R}(x).
\]
Several hand-crafted regularizers have been proposed relying on structural assumptions on the images.
For example, the well-known total variation~\cite{RuOsFa92} enforces sparse gradients and thus induces piecewise constant images.
Multiple extension including the total generalized variation~\cite{BrKu10} and its variants~\cite{KoDo19} penalize higher-order derivatives while promoting piecewise smooth regions.

Recently, several data-driven approaches have been proposed to learn the regularizer embedded in a variational model.
This approach is motivated by the diverse structure of natural images, which is only insufficiently covered by the aforementioned hand-crafted regularizers.
Nowadays, learning-based regularizers vastly outperform their hand-crafted counterparts as predicted in~\cite{LeNa11}, we refer the reader to \cref{sub:relatedWork} for a recent overview.
In this paper, we utilize the total deep variation (TDV) regularizer~\cite{KoEf20} representing a deep multi-scale convolutional neural network, whose robustness and local structure has been thoroughly analyzed in~\cite{KoEf20a}.

While there are various possibilities to design a regularizer, the choice of the data fidelity term depends on the noise statistics and the prior for a given task.
As a classical example, the squared $\ell^2$-norm is the proper choice for additive white Gaussian denoising with a Gaussian prior~\cite{GrNi19}.
However, apart from simulated noise instances and predefined simple priors, the choice of the \emph{optimal} data fidelity term is in general unknown~\cite{Ni07}.
Frequently, $\ell^p$-norms are used due to their simplicity, but their applicability to realistic noise with learned priors is not always justified~\cite{ZhJi20}.

In many practical applications, the statistics of the noise~$\zeta$ in~\eqref{eq:IP} is unknown as well as ground truth images~$y$ are not available for learning.
Hence, in this paper we introduce and analyze a learning framework based on a variational principle, which is tailored for image reconstruction in exactly this real-world context.
In its core, we advocate three components: 
\begin{enumerate}
\item
in \emph{energy-based learning}, both the data fidelity term and the regularizer are learned from data to account for the unknown noise statistics,
\item
the \emph{patch-based Wasserstein loss functional} is designed for energy-based learning in the absence of ground truth images,
\item
in \emph{shared prior learning}, the learned parameters of the prior are shared among the supervised and the patch-based Wasserstein loss functional to further enhance the image quality if no ground truth images are available.
\end{enumerate}
In \emph{energy-based learning}, the data fidelity term is either in the subclass of squared $\ell^2$-data terms, Fr\'echet metrics, or generalized divergences, and the regularizer is TDV.
In supervised learning, all learnable parameters are computed in a mean-field optimal control problem~\cite{EHa19}, in which the state equation coincides with the gradient flow of the energy functional.
The control parameters are given as the entity of all learned parameters of the data fidelity term, the regularizer as well as the stopping time of the gradient flow trajectory.
The terminal state of the gradient flow defines the reconstructed image and the loss penalizes its deviation from the ground truth image in an $\ell^p$-norm.
In numerical experiments, we analyze the particular structure of the learned data fidelity terms in various settings and experimentally prove that we achieve state-of-the-art results for several problems in the supervised case.
A comprehensive overview of learning an energy functional for general applications is presented in~\cite{LeCh06}.

Further, we advocate an unsupervised mean-field optimal control problem with a \emph{patch-based Wasserstein loss functional}, in which a patch-wise statistical comparison of image features is performed.
In its core, the loss functional quantifies the Wasserstein distance of two discrete measures, in which the former again depends on the terminal state of the gradient flow trajectory and the latter relies on a family of uncorrupted reference patches.
We compare three different linear feature extraction operators: the mean-invariant patch extraction operator, the DCT-II operator, and a linear autoencoder.

\emph{Shared prior learning} describes a mean-field optimal control problem, in which the loss functional is a convex combination of the loss functional utilized for supervised learning and the previously defined patch-based Wasserstein loss functional.
We stress that the parameters of the regularizer are \emph{shared} among both functionals, all remaining parameters are optimized separately.
Moreover, we require one data set for each loss functional:
\begin{itemize}
\item
In the unsupervised case, the data set is comprised of \emph{independent} collections of observations and reference patches.
\item
The data set in the supervised loss functional consists of pairs of ground truth and corrupted images, where the latter are synthesized in this work.
Heuristically, the synthesized images should as far as possible exhibit structural similarities with the observations in the unsupervised case.
\end{itemize}
We emphasize that in all tasks we are actually \emph{only} interested in the reconstruction of the observations in the unsupervised case, for which no ground truth images are available.
Moreover, these observations are entirely \emph{unrelated} to the remaining training data.
In the case of realistic image denoising, for instance, the training data set for the supervised task consists of high-quality images and synthesized observations,
the training data set in the unsupervised problem is given by the observed images (corrupted by realistic noise) and a set of uncorrupted reference patches, which are both independent.
Note that this approach shares some similarities with \emph{semi-supervised multi-task regression}~\cite{ZhYe09}.

\medskip

For the numerical realization, we propose three discretization schemes of the gradient flow, which define a discretized mean-field optimal control problem.
The consistency of the different discretization schemes is verified in terms of Mosco convergence, which in particular proves the convergence of the time-discrete minimizers to their time-continuous counterparts. 
In several numerical experiments, we achieve state-of-the-art-results for numerous image reconstruction tasks--even in the absence of ground truth images.

\medskip

The paper is structured as follows:
In \cref{sec:meanfield}, we introduce energy-based learning, the patch-based Wasserstein loss functional and shared prior learning in the time-continuous case.
\Cref{sec:discretization} is devoted to the time-discretization of the gradient flow as well as the discretizations of the data fidelity term and the regularizer.
The consistency of the discretization schemes in terms of Mosco convergence is the subject of \cref{sec:Mosco}.
Finally, we numerically validate the applicability of our approach for various imaging problems in \cref{sec:numerics}, in which we achieve state-of-the-art results.

\subsection{Related Work}\label{sub:relatedWork}
Variational approaches for imaging have been advocated in various publications, among which the TV-$L^2$ model~\cite{RuOsFa92} incorporating the total variation (TV) as the regularizer is one of the most prominent.
However, the total variation relies on the first principle assumption stating that images are composed of piecewise constant regions with sparse gradients, that is why staircasing artifacts are promoted.
This effect can be overcome by the inclusion of higher-order image derivatives~\cite{ChMa00,SeSt08,WuTa10}.
A well-known extension of TV is the total generalized variation (TGV)~\cite{BrKu10}, in which a balancing of the image derivatives up to a predefined order is performed and thus prevents the formation of staircasing.
Different approaches to remove staircasing are based on the inclusion of directional information in the regularizer \cite{BeBu06,LeRo15,PaMa20}
and the penalization of the curvature of level lines to enforce continuity of edges~\cite{NiMu93,ChPo19}.
However, hand-crafted regularizers are unable to entirely capture the statistics of natural images, that is why data-driven methods are commonly superior~\cite{LuOk18,LiSc20}.

\medskip

Throughout the last years, several machine learning-based regularizers have been proposed for image reconstruction~\cite{LuIl18}.
Bigdeli and Zwicker~\cite{BiZw18} exploited autoencoding priors, which achieve competitive results even if the regularizers are trained for similar tasks.
In~\cite{DoWa19}, a deep convolutional network-based prior resulting from an unfolding of the algorithm is proposed, which is composed of denoising modules and successive back-projection modules to enforce data consistency.
A particular generalization of the total variation are Fields of Experts (FoE) regularizers~\cite{RoBl09}, whose building blocks are learned filters and learned potential functions.
In~\cite{SaTa09}, discriminative learning of FoE via implicit differentiation is advocated, which turns out to be superior to the originally proposed generative learning.
Integrating the FoE regularizer in a reaction-diffusion process and unrolling the gradient descent algorithm with iteration-dependent parameters was proposed in~\cite{ChPo17} (TNRD) and \cite{Le16}.
Variational networks~\cite{KoKl17} build upon TNRD by additionally incorporating an incremental proximal gradient scheme.
In~\cite{EfKo19}, an optimal control formulation of the training process is analyzed, in which the state equation is the gradient flow associated with an energy functional composed of a squared $\ell^2$-data term and the FoE regularizer.
Interestingly, the optimal stopping parameter is essential to achieve competitive results for image reconstruction tasks.
Finally, in~\cite{KoEf20,KoEf20a} the training process is modeled as a sampled/mean-field optimal control problem and the FoE regularizer is replaced by the total deep variation (TDV), which is a deep convolutional neural network consisting of several successive U-Net type networks~\cite{RoFi15}.
In various numerical experiments, the state-of-the-art performance for image reconstruction problems is shown and the stability of the regularizer with respect to variations of the input data and learned parameters is validated.

\medskip

In the literature, numerous approaches to unsupervised image restoration have been suggested.
Figueiredo and Leit\~{a}o~\cite{FiLe97} incorporated a compound Gauss--Markov random field to model image discontinuities in combination with the minimum description length principle for unsupervised image restoration.
A recent unsupervised restoration technique is the deep image prior~\cite{UlVe18}, in which untrained convolutional neural networks are fit to a single corrupted image, where the structure of the network is essential for the reconstruction quality.
Based on simple statistical arguments, the Noise2Noise~\cite{LeMu18} method is a further prominent approach for image restoration without requiring uncorrupted data yielding competitive restoration results.
An extension of Noise2Noise is Noise2Void~\cite{KrBu19}, where a single degraded image suffices for the unsupervised training of a denoising network especially designed for medical tasks.
Recently, a further extension called Noise2Self~\cite{BaRo19} was advocated, in which a self-supervised loss functional is exploited for image denoising.
Du et al.~\cite{DuCh20} introduced a novel general unsupervised learning method based on invariant representations of noise data exploiting an adversarial domain adaption.
Pajot et al.~\cite{PaBe19} advocated an unsupervised image reconstruction framework, where the loss function is a linear combination of an adversarial loss and a reconstruction loss to enforce data consistency.
Recently, Dittmer et al.~\cite{DiSc20} devised an unsupervised framework for additive noise removal based on Wasserstein GANs \cite{ArCh17}.

There are several links to our particular Wasserstein-based loss functional.
In \cite{EnRe20}, the performance of the quadratic Wasserstein distance for data matching is analyzed and the robustness against high-frequency noise patterns is shown.
Schmitz et al.~\cite{ScHe18} advocated a nonlinear Wasserstein dictionary learning for images decoded as probability measures.
In a similar fashion, Rolet et al.~\cite{RoCu16} proposed a dictionary learning framework, in which the observations are encoded as normalized histograms of features.

To the best of our knowledge, there has been no technique comparable with the proposed shared prior learning in the literature.
However, the combination of (simulated) training data with natural image data in a semi-supervised fashion has been analyzed in several publications, which mostly rely on elaborate convolutional neural network structures.
Recent publications in this direction address the tasks of low-dose computed tomography reconstruction~\cite{LiYe19}, single image rain removal~\cite{WeMe19}, and image dehazing~\cite{LiDo20}.
The inclusion of unpaired training data has also been utilized for the training of generative adversarial networks~\cite{ZhPa17,ChVa19}.

\subsection{Notation}
We denote by $C^0(X,Y)$ the space of continuous, by $C^k(X,Y)$ the space of $k$-times continuously differentiable functions
and by $C^{k,\alpha}(X,Y)$, $\alpha\in(0,1]$, the space of H\"older functions mapping from~$X$ to~$Y$, where an additional subscript $c$ indicates a function with compact support.
The associated norms are denoted by $\Vert\cdot\Vert_{C^0}$, $\Vert\cdot\Vert_{C^k}$, $\Vert\cdot\Vert_{C^{k,\alpha}}$, respectively.
In addition, $\vert\cdot\vert_{C^{k,\alpha}}$ is the seminorm in the respective H\"older space.
Further, we use the standard notation $L^p(X)$ and $H^m(X)=W^{m,2}(X)$ to denote Lebesgue and Sobolev spaces, respectively.
The symbol $\mathbf{1}_n\in\R^n$ denotes the column one vector, i.e.~$\mathbf{1}_n=(1,\ldots,1)^\top$, and $\Id$ is the identity matrix in the respective vector space.
The largest eigenvalue of a matrix~$A$ is denoted by $\sigma_{\max}(A)$.
Finally, the indicator function of a set~$S$ is written as~$\mathbb{I}_S$, i.e.~$\mathbb{I}_S(x)=1$ if $x\in S$ and $0$ otherwise.

\subsection{Wasserstein distance}\label{sub:Wasserstein}
In this section, we recall the Wasserstein distance in the general case as well as a fully discrete version.
Afterwards, we utilize an algorithm to compute the discrete Wasserstein distance incorporating the Bregman distance on the entropy, which can be regarded as a modification of the widely used Sinkhorn algorithm~\cite{Si64,Cu13}.

The definition of the Wasserstein distance is as follows~\cite{Vi09,PeCu19}:
\begin{definition}
Let $\mu,\nu$ be two probability measures on~$\R^n$.
Then, for $p\in[1,\infty)$ the \emph{Wasserstein distance} associated with the $\ell^p$-norm is given by
\[
\Wasserstein_p(\mu,\nu)\coloneqq\inf_{\pi\in\Pi(\mu,\nu)}\int_{\R^n\times\R^n}\Vert x-y\Vert_p\dx\pi(x,y),
\]
where $\Pi(\mu,\nu)$ denotes the set of all joint probability measures on~$\R^n\times\R^n$ whose marginals are~$\mu$ and~$\nu$.
\end{definition}
In what follows, we define for any collection of points $v=(v_1,\ldots,v_N)$, $v_i\in\R^n$, the \emph{discrete measure}~$\mu[v]$ as
\[
\mu[v](S)\coloneqq\frac{1}{N}\sum_{i=1}^N\delta_{v_i}(S)
\]
for $S\subset\R^n$, where $\delta_{v_i}(S)=1$ if $v_i\in S$ and~$0$ otherwise.
For a second collection of points $w=(w_1,\ldots,w_N)$, $w_i\in\R^n$, and $p\in[1,\infty)$ the \emph{cost matrix}~$C_p[{v,w}]\in\R^{N\times N}$ is defined as
\[
(C_p[{v,w}])_{i,j}\coloneqq\Vert v_i-w_j\Vert_p.
\]
Following~\cite[Chapter~6.4]{Sa15}, the \emph{discrete Wasserstein distance} between $\mu[v]$ and $\mu[w]$ with cost matrix~$C_p[{v,w}]$ reads as
\[
\WassersteinDiscrete_p(\mu[v], \mu[w])\coloneqq\inf\left\{\tr((C_p[v,w])^\top P)\!:\!P\in\R^{N\times N}, P_{i,j}\geq 0,P^\top\mathbf{1}_N=\tfrac{1}{N}\mathbf{1}_N,P\mathbf{1}_N=\tfrac{1}{N}\mathbf{1}_N\right\}.
\]
Classically, the Sinkhorn algorithm is exploited to approximate the discrete Wasserstein distance using an entropy regularization.
In this paper, we apply a proximal version of the Sinkhorn algorithm proposed in~\cite{BeCa15,XiWa19}, which is based on an inexact proximal point method to increase convergence speed and stability.
In fact, the Bregman divergence with respect to the entropy and the solution of the transport map in the previous iteration step is added to the discrete Wasserstein distance.
This reformulation amounts to a rescaling of the cost matrix in each iteration step.
Thus, compared to the classical Sinkhorn algorithm, the only difference can be found in line~\ref{state:Q} of~\Cref{alg:Wasserstein}.
\begin{algorithm}[htb]
\begin{algorithmic}[1]
\STATE{\textbf{Initial}: discrete probability measures $\mu[v]$ and $\mu[w]$, collections of points $v$ and $w$, cost matrix\\ $C_p[{v,w}]$ for $p\geq 1$, regularization parameter~$\beta>0$, maximum number of iterations~$J\in\N$}
\STATE{$b=\frac{1}{N}\mathbf{1}_N$}
\STATE{$G_{i,j}=\exp(-\frac{1}{\beta}(C_p[{v,w}])_{i,j})$}
\STATE{$T_N^{(1)}=\mathbf{1}_N\mathbf{1}_N^\top$}
\FOR{$j=1$ \TO $J$}
\STATE{$Q=G\odot T_N^{(j)}$}\label{state:Q}
\STATE{$a=\frac{\mu[v]}{Qb}$}
\STATE{$b=\frac{\mu[w]}{Q^\top a}$}
\STATE{$T_N^{(j+1)}=\diag(a)Q\diag(b)$}
\ENDFOR
\RETURN $T_N^{(J+1)}$
\end{algorithmic}
\caption{Proximal version of Sinkhorn algorithm to approximate $\WassersteinDiscrete_p$.}
\label{alg:Wasserstein}
\end{algorithm}

\section{Time-continuous mean-field optimal control problems}\label{sec:meanfield}
In this section, we introduce all novel main concepts of this paper: energy-based learning (\cref{sub:variationalLearning}), the patch-based Wasserstein loss functional (\cref{sub:patchWasserstein}), and shared prior learning (\cref{sub:sharedPriorLearning}).

In this paper, we are concerned with inverse problems~\cite{EnHa96} of the form
\begin{equation}
z=Z(Ay,\zeta),
\label{eq:inverseProblem}
\end{equation}
where $y\in\R^{n_y}$ denotes an unknown ground truth image with a resolution of $n=\mathrm{width}\times\mathrm{height}$ and $C$~channels, i.e.~$n_y=nC$.
Throughout this paper, we always identify images with their vector representations.
As an example, the ground truth image on the multichannel pixel grid $\R^{\mathrm{width}\times\mathrm{height}\times C}$ is cast as~$\R^{n_y}$.
Moreover, $z\in\R^{n_z}$ refers to the observation, and $\zeta\in\R^{n_z}$ is a random variable modeling i.i.d.~noise with a finite second momentum.
We assume that the task-dependent linear operator $A\in\R^{n_z\times n_y}$ and the observation-generating function $Z:\R^{n_z}\times\R^{n_z}\to\R^{n_z}$ are fixed.
For simplicity, we only consider two-dimensional data and remark that the subsequent methods are applicable to data of any dimension with minor modifications.
In the case of additive white Gaussian noise, for instance, $A$ is the identity operator with $n_y=n_z$, $\zeta\sim\mathcal{N}(0,\sigma^2\Id)$ and $Z(y,\zeta)=y+\zeta$.

\subsection{Energy-based learning}\label{sub:variationalLearning}
The starting point of our analysis presented in \cref{subsub:variation} is the variational formulation of inverse problems, in which the energy functional is composed of a learned data fidelity term and a learned regularizer.
In \cref{subsub:supervised}, the supervised learning process is modeled as a mean-field optimal control problem, where the state equation is the gradient flow associated with the energy functional.

\subsubsection{Energy functional and gradient flow}\label{subsub:variation}
In energy-based learning, both the data fidelity term and the regularizer are learned from data.
We retrieve a reconstruction~$x\in\R^{n_y}$ of~$y$ by minimizing
\begin{equation}
x\in\argmin_{\widetilde{x}\in\R^{n_y}}\left\{\mathcal{E}(\widetilde{x},z,\xi,\theta)\coloneqq\mathcal{D}(A\widetilde{x},z,\xi)+\mathcal{R}(\widetilde{x},\theta)\right\}.
\label{eq:variationalProblem}
\end{equation}
The data fidelity term $\mathcal{D}\in C_c^3(\R^{n_z}\times\R^{n_z}\times\Xi,\R_0^+)$ with compact support quantifies the distance of~$Ax$ from the observation~$z\in\R^{n_z}$.
Here, $\xi\in\Xi$ denotes the learned parameter contained in the compact \emph{data parameter space} $\Xi\subset\R^{n_\Xi}$.
Throughout this paper, we consider three different choices of~$\mathcal{D}$, where the compact support is enforced by a suitable smooth truncation for sufficiently large values:
\begin{enumerate}
\item
Following~\cite{KoEf20,KoEf20a}, we consider the \emph{squared scaled $\ell^2$-data term} given by
\[
\mathcal{D}(Ax,z,\xi)=\mathcal{D}_{\ell^2}(Ax,z,\xi)\coloneqq\frac{\xi}{2}\Vert Ax-z\Vert_2^2,
\]
where the scale parameter~$\xi\in\Xi$ for $n_\Xi=1$ is learned from data.
\item 
Next, we consider the case in which $\mathcal{D}$ is a \emph{Fr\'echet metric}, i.e.
\[
\mathcal{D}(Ax,z,\xi)=\mathcal{D}_F(Ax,z,\xi)\coloneqq\sum_{i=1}^{n_z}\rho_F((Ax-z)_i,\xi),
\]
where $\rho_F\in C_c^3(\R\times\Xi,\R_0^+$) is a function parametrized by $\xi$ satisfying
\begin{enumerate}
\item 
$\rho_F(x,\xi)=\rho_F(-x,\xi)$,
\item
$\rho_F(x,\xi)\geq0$, where $\rho_F(x,\xi)=0$ if and only if $x=0$,
\item 
$\rho_F(x+y,\xi)\leq\rho_F(x,\xi)+\rho_F(y,\xi)$.
\end{enumerate}
\item
In the case of a \emph{generalized divergence}, we assume that
\[
\mathcal{D}(Ax,z,\xi)=\mathcal{D}_\div(Ax,z,\xi)\coloneqq\sum_{i=1}^{n_z}\rho_\div((Ax)_i,z_i,\xi).
\]
Here, $\rho_\div\in C_c^3(\R\times\R\times\Xi,\R^+_0)$ is convex in the first argument and satisfies $\rho_\div(s,t,\xi)=0$ if and only if $s=t$.
\end{enumerate}
Throughout this paper, the regularizer $\mathcal{R}\in C^3_c(\R^{n_y}\times\Theta,\R)$ is the \emph{total deep variation}~\cite{KoEf20,KoEf20a}, which is a convolutional neural network with learned parameters~$\theta\in\Theta$ for a compact \emph{regularizer data space} $\Theta\subset\R^{n_\Theta}$.
Here, $\mathcal{R}$ is the sum of the pixelwise deep variation~$r:\R^{n_y}\times\Theta\to\R^n$, i.e.
$\mathcal{R}(x,\theta)=\sum_{i=1}^n r(x,\theta)_i$, where $r$ has the form $r(x,\theta)=w\mathcal{N}(Kx)$.
The components of the pixelwise deep variation are as follows:
\begin{itemize}
\item
$K\in\R^{mn\times n_y}$ is the matrix representation of a learned $3\times3$ convolution kernel~$k$ with $m$~feature channels and zero-mean constraint (i.e.~$\sum_{i,j=1}^{3}k_{i,j}=0$) to enforce mean-invariance,
\item
$\mathcal{N}:\R^{mn}\to\R^{mn}$ is a sufficiently smooth multiscale convolutional neural network with compact support, which is detailed in~\cref{sub:TDV},
\item
$w\in\R^{n\times mn}$ is the matrix representation of a learned $1\times1$ convolution kernel.
\end{itemize}
Thus, each $\theta$ encodes the learnable parameters of~$K$, $\mathcal{N}$ and $w$.

\medskip

A widespread approach to minimize~\eqref{eq:variationalProblem} relies on the \emph{gradient flow}~\cite{AmGi08}, which reads as
\begin{equation}
\dot{\widetilde{x}}(t)=f(\widetilde{x}(t),z,\xi,\theta)\coloneqq-\nabla_1\mathcal{E}(\widetilde{x}(t),z,\xi,\theta)
=-A^\top\nabla_1\mathcal{D}(A\widetilde{x}(t),z,\xi)-\nabla_1\mathcal{R}(\widetilde{x}(t),\theta),
\label{eq:originalGradientFlow}
\end{equation}
where $t$ is contained in the time interval~$[0,T]$ and $T$ is the \emph{stopping time}.
The initial value~$\widetilde{x}(0)$ is set to $A_\init z$ for a fixed task-dependent matrix $A_\init\in\R^{n_y\times n_z}$.
Further, we assume that $T\leq\Tmax$ with $\Tmax>0$ denoting the \emph{maximum time horizon}.
The reparametrization $x(t)=\widetilde{x}(tT)$ converts \eqref{eq:originalGradientFlow} into the equivalent gradient flow
\begin{equation}
\dot{x}(t)=Tf(x(t),z,\xi,\theta)
\label{eq:gradientFlow}
\end{equation}
on the time interval~$[0,1]$ with the same initial value as before.
In this case, the reconstruction of the unknown ground truth image is given by~$x(1)$.
In general, we denote the image trajectory evaluated at time $t\in[0,1]$ by $x(t,z,T,\xi,\theta)$ to highlight the dependency on $(z,T,\xi,\theta)$.

\subsubsection{Optimal control problem for supervised energy-based learning}\label{subsub:supervised}
In this section, we cast the training process for the supervised energy-based learning as a mean-field optimal control problem following~\cite{E17,EHa19}.
To this end, let $(\Omega^\supervised,\mathcal{F}^\supervised,\Prob^\supervised)$ be a complete probability space,
which is associated with the joint data distribution~$\training^\supervised$ of the training data $(\overline y,\overline z)\in\R^{n_y}\times\R^{n_z}$ representing pairs of ground truth images and corresponding observations.
We assume that both random variables are \emph{dependent} with a finite second momentum, and we denote by $\training^{\supervised,\mathcal{Z}}$ the marginal distribution with respect to the observations.

Commonly, the learned parameters in a machine learning setting are optimized by comparing the outcome of a network depending on input data with the associated ground truth in a specific metric.
In our case, given a pair~$(\overline y,\overline z)\sim\training^\supervised$ we are aiming at minimizing the mean distance of
the reconstruction~$x(1,\overline z,T,\xi,\theta)$ from the ground truth image~$\overline y$ measured in terms of the \emph{loss function}~$\loss$, which is either $\loss(x,y)=(\sum_i(x_i-y_i)^2+\iota^2)^\frac{1}{2}$ (denoted by $\ell_\iota^1$) for $\iota>0$ or $\loss(x,y)=\Vert y-x\Vert_2$ (denoted by $\ell^2$).
Then, the optimal parameters $(T,\xi,\theta)$ that uniquely determine the model output $x(1,\overline z,T,\xi,\theta)$ are inferred from
the \emph{mean-field optimal control problem for supervised energy-based learning} given by
\begin{equation}
\inf
\left\{\E_{(\overline y,\overline z)\sim\training^\supervised}\loss(x(1,\overline z,T,\xi,\theta),\overline y):T\in[0,\Tmax],\xi\in\Xi,\theta\in\Theta\right\}.
\label{eq:objectiveFunctionSupervised}
\end{equation}
We remark that the state equation~\eqref{eq:gradientFlow} of this optimal control problem is a stochastic ordinary differential equation whose only source of randomness is the observation~$z$.
The existence of optimal control parameters~$(T,\xi,\theta)$ for~\eqref{eq:objectiveFunctionSupervised} is discussed in~\cref{sub:sharedPriorLearning}.
\Cref{fig:supervisedLearning} visualizes the training process in the mean-field optimal control setting for salt-and-pepper denoising with $A=A_\init=\Id$ using the shorthand notation $x(t,\overline z)=x(t,\overline z,T,\xi,\theta)$.
Compared to the mean-field optimal control problem proposed in~\cite{KoEf20a}, we here additionally learn the parameters~$\xi$ of the data fidelity term, which gives rise to the term \emph{energy-based learning}.
\begin{figure}[htb]
\begin{center}
\includegraphics[width=.9\linewidth]{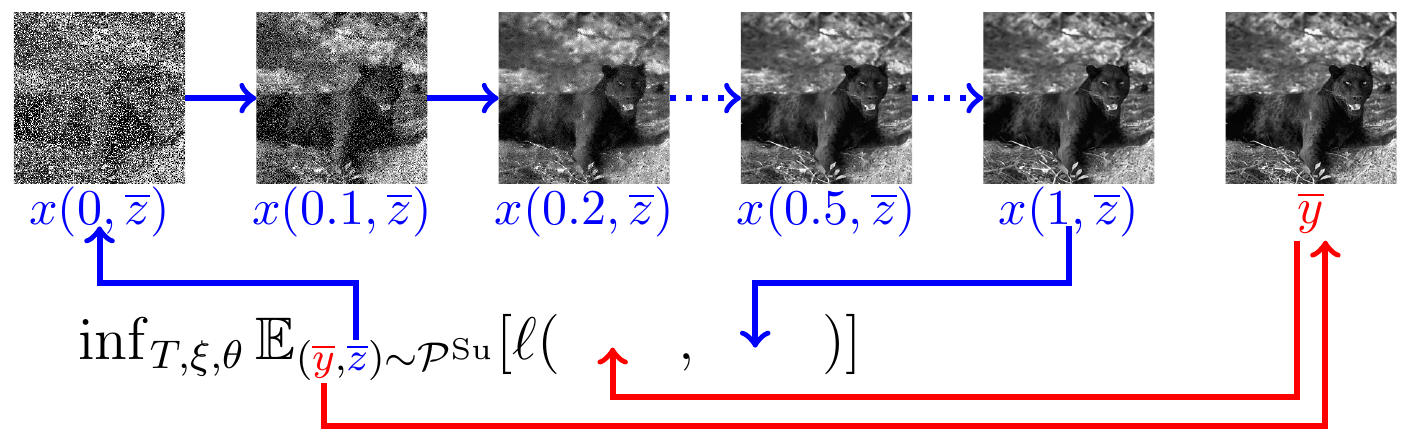}
\end{center}
\caption{Visualization of the supervised mean-field optimal control training problem.}
\label{fig:supervisedLearning}
\end{figure}
    
\subsection{Patch-based Wasserstein loss functional}\label{sub:patchWasserstein}
In this section, we propose an unsupervised mean-field optimal control problem based on the patch-wise Wasserstein distance to compare local image statistics.
As before, the state equation coincides with the gradient flow of the previously defined energy functional.

Motivated by the pioneering work of Mumford et al.~\cite{HuMu99,MuGi01,LePe03} on natural image statistics, we intend to compare the statistics of the approximate reconstruction given by the terminal state of the state equation with randomly drawn patches of reference images.
To this end, we propose a cost functional which quantifies the mismatch of the distributions of the reconstructed image and reference patches using the discrete Wasserstein distance.

The comparison of whole images has proven to be infeasible with current computational tools due to the curse of dimensionality~\cite{PeCu19}.
To overcome this issue, we restrict our method to local image comparisons on the level of patches.
Therefore, we randomly draw $N\in\N$ patches of size $n_\patch^2$, which is usually larger than the size of the reconstructed image since we allow overlapping patches.
To model the training data set, we consider the complete probability space $(\Omega^\unsupervised,\mathcal{F}^\unsupervised,\Prob^\unsupervised)$ defining the distribution~$\training^\unsupervised$,
where each pair of random variables $(\widehat z,\widehat p)\sim\training^\unsupervised$ models an observation~$\widehat z\in\mathcal{Z}^\unsupervised\subset\R^{n_z}$ and a collection of reference patches~$\widehat p\in\R^{N\times n_\patch^2}$.
We assume that both random variables are \emph{independent} with a finite second momentum.
Further, we denote by $\training^{\unsupervised,\mathcal{Z}}$ the corresponding marginal distribution with respect to the observations.

In what follows, we conduct a statistical comparison of the observations and reference patches on the level of features, which encode local structure information.
The \emph{patch extraction operator} $\Patch:\R^{n_y}\to\R^{N\times n_\patch^2}$ crops the image into $N$~overlapping patches of size~$n_\patch^2$.
The features are extracted using a linear \emph{feature extraction operator} $\F:\R^{N\times n_\patch^2}\to\R^{N\times n_\F}$.
In this work, we consider the subsequent choices for the feature extraction operator~$\F$:
\begin{enumerate}
\item[$\ID$:]
Let $\R^{N\times n_\F}/\sim$ be the equivalence class of $N$~square patches with width and height~$n_\patch$ (i.e.~$n_\F=n_\patch^2$), in which two patches are equivalent if and only if they coincide after subtracting the respective patch means.
The mean-invariant identity operator~$\ID:\R^{N\times n_\patch^2}\to\R^{N\times n_\F}/\sim$ subtracts the mean without further alterations of the patches.
\item[$\DCT$:]
As advocated by~\cite{LePe03}, a reasonable feature extraction operator is the DCT-II transform without the constant component~\cite{BrYi06}.
In our case, the DCT~operator is a mapping $\DCT:\R^{N\times n_\patch^2}\to\R^{N\times n_\F}$ with $n_\F=n_\patch^2-1$, which assigns DCT coefficients to each patch, where the first coefficient is neglected to obtain a mean-invariant representation.
\item[$\autoencoder$:]
Finally, we utilize the linear autoencoder operator $\autoencoder:\R^{N\times n_\patch^2}\to\R^{N\times n_\F}$ with $n_\F=n_\patch^2-1$ (in analogy to~$\DCT$) as a feature extraction operator, which is pretrained on the BSDS400 data set.
In detail, let $x_i\in\R^{n_y}$, $1\leq i\leq 400$, denote the collection of training images of the BSDS400 data set and $x_i^j\in\R^{n_\patch^2}$ be the $j^{th}$~square patch of $x_i$ for $1\leq j\leq N$ with zero mean constraint, which is realized by applying $\ID\circ\Patch$.
Hence, the operator $\autoencoder=\underbrace{(\widetilde{\autoencoder},\ldots,\widetilde{\autoencoder})}_{N\text{ times}}\circ\ID$ is obtained from
\[
\min_{\widetilde{\autoencoder}\in\R^{n_\F\times n_\patch^2},\mathbf{D}\in\R^{n_\patch^2\times n_\F}}\sum_{i=1}^{400}\sum_{j=1}^N\Vert\mathbf{D}(\widetilde{\autoencoder}x_i^j)-x_i^j\Vert_2^2.
\]
In other words, $\autoencoder$ is composed of $N$~identical operators $\widetilde{\autoencoder}$ mapping each single patch to a lower-dimensional representation.
\end{enumerate}
Next, we propose the \emph{mean-field optimal control problem for unsupervised learning}.
In its core, the Wasserstein distance of the two discrete measures $\mu[\F(\Patch x(1,\widehat z,T,\xi,\theta))]$ and $\mu[\F\widehat p]$ is minimized in the mean-field setting, i.e.~$\widehat z$ and~$\widehat p$ are random variables.
Note that $x(1,\widehat z,T,\xi,\theta)$ denotes the terminal state of the gradient flow equation emanating from the observation~$\widehat z$ with the control parameters $T$, $\xi$ and~$\theta$.
Henceforth, we use the abbreviation
\[
\WassersteinDiscrete_{\F,p}(v,w)\coloneqq\WassersteinDiscrete_p(\mu[\F v],\mu[\F w]).
\]
Thus, the optimal control problem reads as
\begin{equation}
\inf\left\{\E_{(\widehat z,\widehat p)\sim\training^\unsupervised}\WassersteinDiscrete_{\F,p}(\Patch x(1,\widehat z,T,\xi,\theta),\widehat p):T\in[0,\Tmax],\xi\in\Xi,\theta\in\Theta\right\}.
\label{eq:objectiveFunctionUnsupervised}
\end{equation}
For the definitions of the discrete measures, the cost matrix and the discrete Wasserstein distance, we refer the reader to~\cref{sub:Wasserstein}.
We prove the existence of optimal control parameters~$(T,\xi,\theta)$ associated with~\eqref{eq:objectiveFunctionUnsupervised} in~\cref{sub:sharedPriorLearning}.
\Cref{fig:unsupervisedLearning} summarizes the unsupervised mean-field optimal control problem (using the notation $x(t,\widehat z)=x(t,\widehat z,T,\xi,\theta)$).
\begin{figure}[htb]
\begin{center}
\includegraphics[width=.9\linewidth]{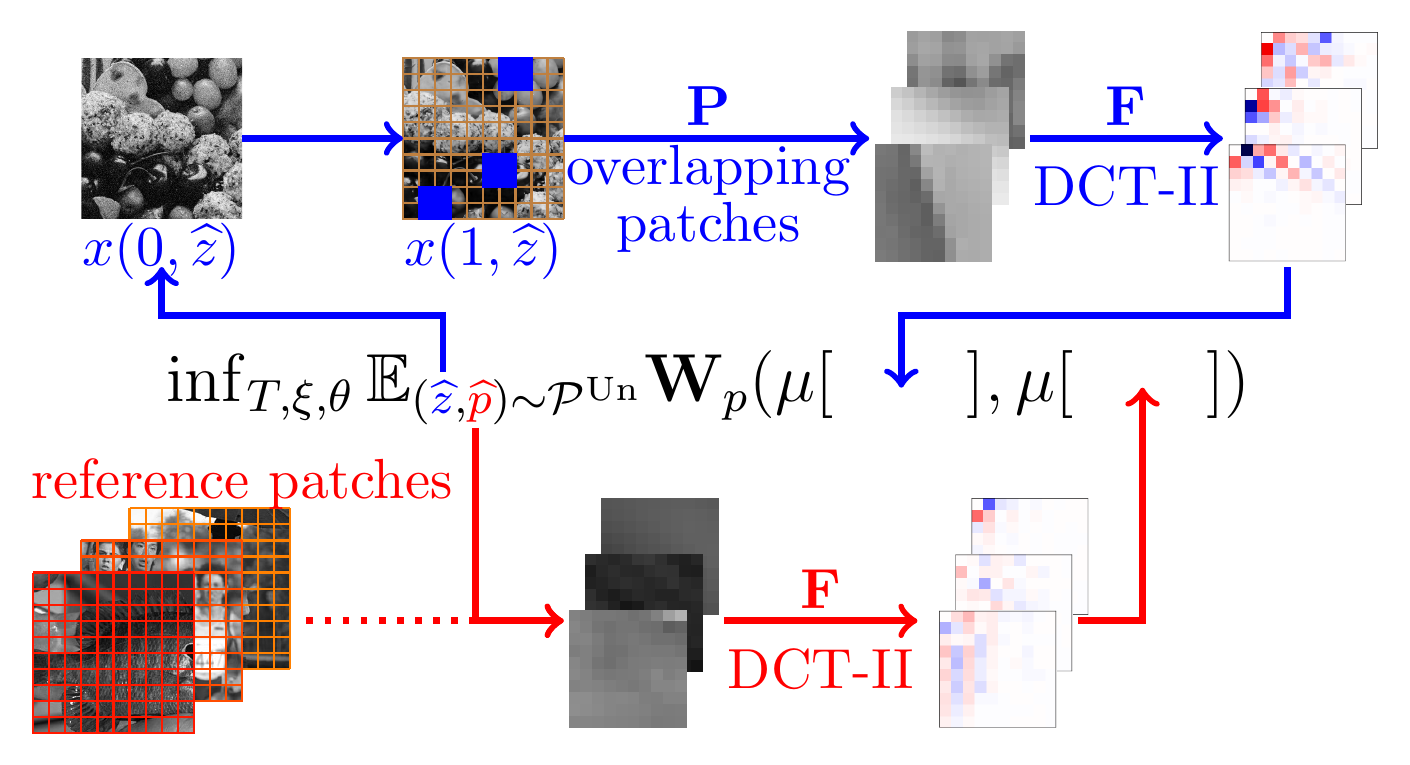}
\end{center}
\caption{Visualization of the unsupervised mean-field optimal control training problem.}
\label{fig:unsupervisedLearning}
\end{figure}

\begin{figure}
\includegraphics[width=\linewidth]{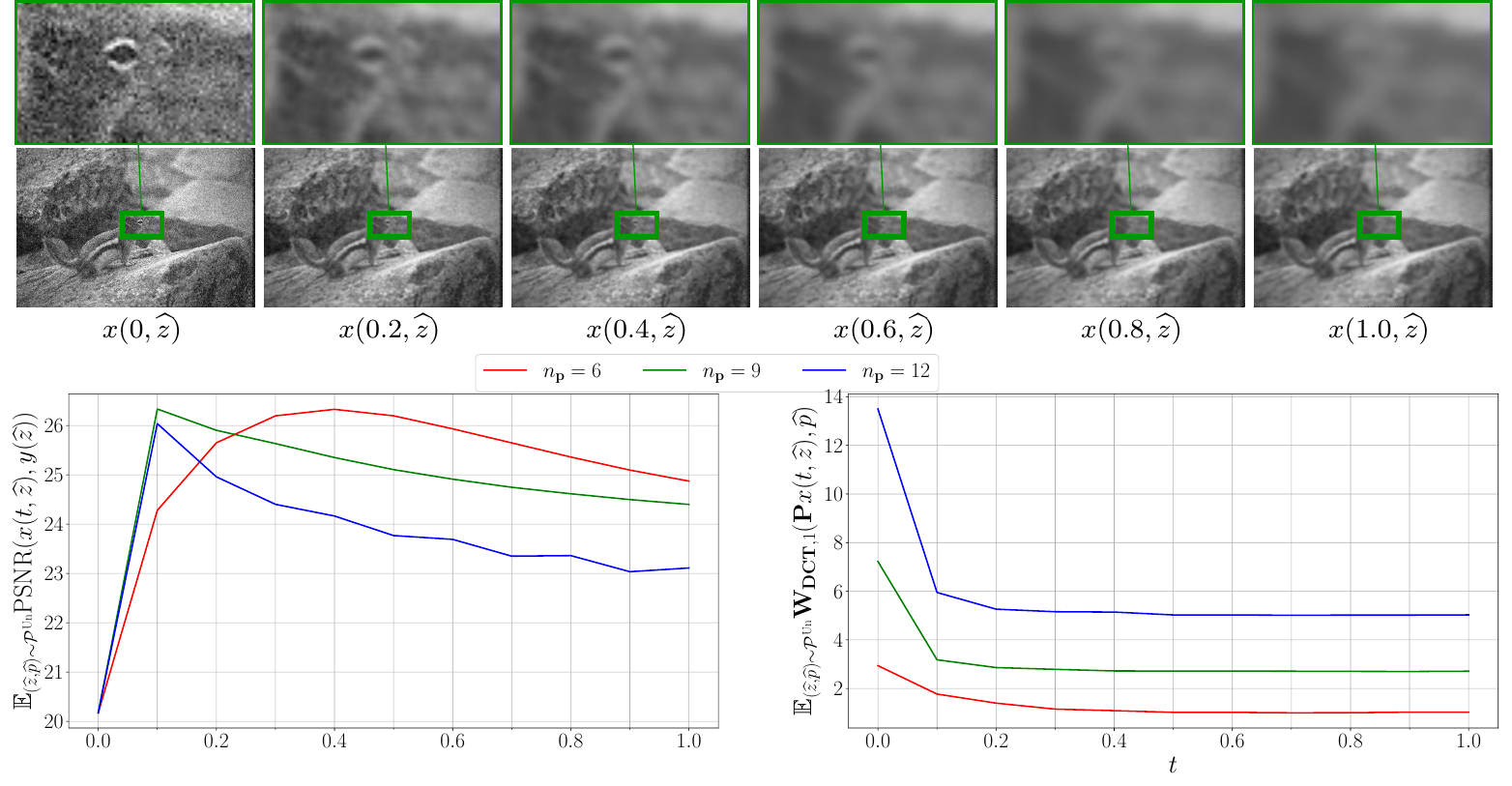}
\caption{Numerical experiments to illustrate pros and cons related to the patch-based Wasserstein loss functional.}
\label{fig:WassersteinRelevance}
\end{figure}

There are several pros and cons of the sole use of the patch-based Wasserstein loss functional.
For illustration, \Cref{fig:WassersteinRelevance} (first row) visualizes a prototypic image trajectory for additive white Gaussian denoising ($n_\patch=6$, $\sigma=25$, $A=A_\init=\Id$) along with a zoom with factor~$6$, which is evaluated at $t\in\{0,0.1,\ldots,1\}$.
All control parameters are computed using the patch-based Wasserstein loss functional in~\eqref{eq:objectiveFunctionUnsupervised}.
The probability measure associated with~$\training^\unsupervised$  is the discrete measure related to image patches of size $99\times 99$ of the DIV2K data set~\cite{AgTi17} for training and the BSDS68 data set~\cite{MaFo01} for validation, which were corrupted by additive white Gaussian noise with $\sigma=25$,
and reference patches~$\widehat p$ of the BSDS400 data set~\cite{MaFo01} of size $n_\patch^2$.
The initial noisy image is gradually smoothed along the trajectory, which results in a blurry image at its terminal point.
This property can also be observed in the entire data set.
In detail, \Cref{fig:WassersteinRelevance} (second row) shows the plots 
\[
t\mapsto\E_{(\widehat z,\widehat p)\sim\training^\unsupervised}\PSNR(x(t,\widehat z),y(\widehat z)),\qquad
t\mapsto\E_{(\widehat z,\widehat p)\sim\training^\unsupervised}\WassersteinDiscrete_{\DCT,1}(\Patch x(t,\widehat z),\widehat p)
\]
for $t\in\{0,0.1,\ldots,1\}$ and different color-coded $n_\patch\in\{6,9,12\}$.
Here, $y(\widehat z)$ denotes the ground truth image associated with~$\widehat z$.
As a result, all $\PSNR$ curves have a peak for $t\in[0.1,0.4]$, whereas the minimum of the Wasserstein distances is attained for larger~$t$.
Nevertheless, a significant increase of all $\PSNR$ scores compared to the noisy initialization can be observed.
Moreover, high-frequency information are discarded and oversmoothed images are generated for larger~$t$.

\subsection{Shared prior learning}\label{sub:sharedPriorLearning}
Next, we introduce shared prior learning as an extension of the patch-based Wasserstein loss functional, for which no ground truth images are required.

Shared prior learning describes a convex combination of the supervised loss functional with the patch-based Wasserstein loss functional resulting in the \emph{mean-field optimal control problem}
\begin{equation}
\inf\left\{J(x(1,\cdot,T^\supervised,\xi^\supervised,\theta),x(1,\cdot,T^\unsupervised,\xi^\unsupervised,\theta)):(T^\supervised,T^\unsupervised,\xi^\supervised,\xi^\unsupervised,\theta)\in\Gamma\right\},
\label{eq:objectiveFunctionSemi}
\end{equation}
where the cost functional is given by
\begin{equation}
J(\overline x,\widehat x)\coloneqq\alpha\E_{(\overline y,\overline z)\sim\training^\supervised}\loss(\overline x(\overline{z}),\overline y)+(1-\alpha)\E_{(\widehat z,\widehat p)\sim\training^\unsupervised}\WassersteinDiscrete_{\F,p}(\Patch \widehat x(\widehat z),\widehat p)
\label{eq:costFunctional}
\end{equation}
for $\alpha\in[0,1]$.
Here, the set of all control parameters is denoted by
\[
\Gamma\coloneqq[0,\Tmax]^2\times\Xi^2\times\Theta.
\]
Note that the optimal control problem~\eqref{eq:objectiveFunctionSemi} coincides with the supervised optimal control problem for~$\alpha=1$ and with the unsupervised optimal control problem for~$\alpha=0$.
We stress that the stopping times~$T^\supervised$ and~$T^\unsupervised$ as well as the parameters~$\xi^\supervised$ and~$\xi^\unsupervised$ of the data fidelity terms are
learned individually in the supervised and the unsupervised loss functionals, whereas the regularization parameter~$\theta$ is shared among both.
This separation is motivated by the fact that the data fidelity terms strongly depend on the reconstruction task, while the regularizer reflects the prior knowledge of the underlying image distribution.

In the standard setting, $(\overline y,\overline z)\sim\training^\supervised$ describes uncorrupted images~$\overline y$ taken from any data set, where the corresponding observations~$\overline z$ are synthesized by the known forward model~\eqref{eq:inverseProblem} potentially degraded by simulated noise,
and $(\widehat z,\widehat p)\sim\training^\unsupervised$ refers to pairs of observations~$\widehat z$ generated by a forward model with realistic noise and \emph{independent} reference patches~$\widehat p$.
We emphasize that the observations~$\widehat z$ are entirely \emph{unrelated} to all remaining training data, which are corrupted by unknown noise and potentially by an unknown linear operator in the forward model~\eqref{eq:inverseProblem}.
Actually, we are \emph{exclusively} interested in the reconstruction of the observations~$\widehat z$ and not in the reconstruction quality for the synthesized observations~$\overline z$.
Intuitively, the reconstruction quality of our method depends on the similarity of the degradation processes~\eqref{eq:inverseProblem} for the synthesized and the real observations.
Hence, shared prior learning can be seen as a trade-off between two opposing effects:
\begin{itemize}
\item
Supervised learning leads to impressive reconstruction results if the noise statistics and the linear operator are known.
However, if these assumptions are not entirely satisfied (for example, if the noise statistics are slightly altered), then supervised learning could yield poor results.
\item
We have seen in \cref{sub:patchWasserstein} that unsupervised learning accurately compares image statistics, but fails to recover high-frequency information.
\end{itemize}
In summary, shared prior learning is an approach to further enhance the reconstruction quality of the given observations~$\widehat z$ without ground truth, to which
all remaining training data~$\overline y$, $\overline z$ and~$\widehat p$ are unrelated.
The shared parameters of the prior are simultaneously adapted to the real images via the unsupervised loss functional and to the synthesized observations, and thereby we intend to overcome the disadvantages of the patch-based Wasserstein loss functional discussed in \cref{sub:patchWasserstein}.

Next, we numerically verify the benefits of shared prior learning.
Let $\training^\supervised$ be the distribution modeling the entire BSDS400 data set~\cite{MaFo01} (each image with equal probability),
where each $\overline z$ is the sum of a ground truth image~$\overline y$ and additive white Gaussian noise with $\sigma=15$.
The probability measure in the unsupervised case coincides with the measure used for \Cref{fig:WassersteinRelevance}, in which the observations are degraded by additive white Gaussian noise with~$\sigma=25$.
The image trajectory depicted in \Cref{fig:sharedPriorOverview} (first row) emanates from the same noisy initial image as in \Cref{fig:WassersteinRelevance},
all subsequent images are evaluated at $t\in\{0.2,0.4,\ldots,1.0\}$, for which the control parameters are computed in~\eqref{eq:objectiveFunctionSemi} using $\alpha=0.9$.
Compared to the corresponding sequence in \Cref{fig:WassersteinRelevance}, fine details and structures are preserved and less smoothing artifacts are visible.
We have computed the average $\PSNR$ values on the entire validation set in \Cref{fig:sharedPriorOverview} (lower left), in which the control parameters are obtained from \eqref{eq:objectiveFunctionSemi} with $\alpha\in\{0,0.5,0.9,1.0\}$.
At the terminal point, the average $\PSNR$ scores of the curves $\alpha=0.5$ and $\alpha=0.9$ are significantly above the corresponding unsupervised curve, which itself lies clearly above the supervised curve.
The corresponding Wasserstein distances are plotted in~\Cref{fig:sharedPriorOverview} (lower right).
Interestingly, the terminal points of the curves $\alpha\in\{0,0.5,0.9\}$ nearly coincide, whereas the terminal point of the supervised curve with $\alpha=1$ is substantially higher.
In summary, this example nicely illustrates the superiority of shared prior learning to pure supervised or unsupervised learning in the absence of ground truth images.
We will present further numerical experiments in \cref{sec:numerics}.

\begin{figure}
\includegraphics[width=\linewidth]{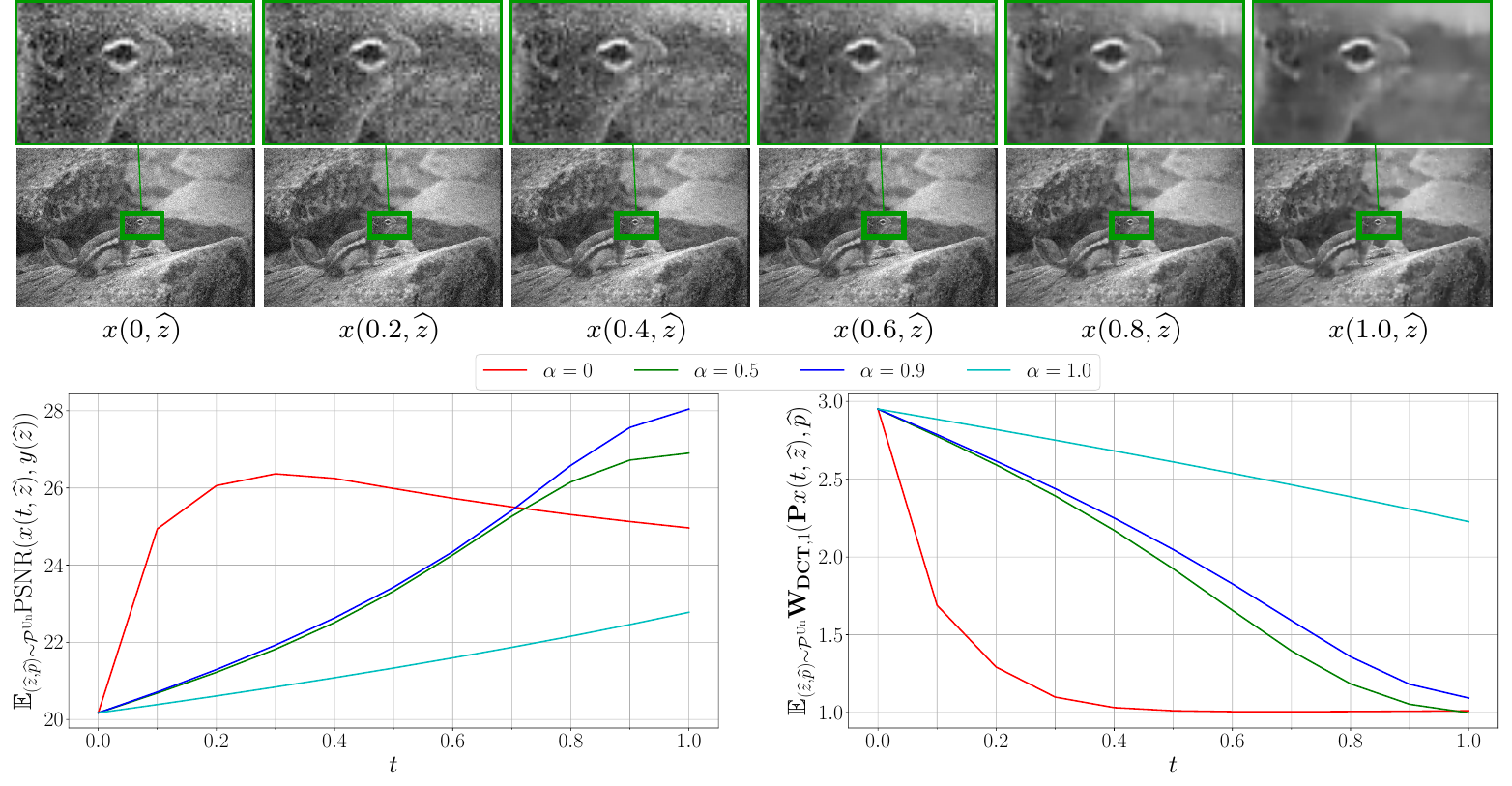}
\caption{Illustration of the advantages of shared prior learning compared to the sole use of the patch-based Wasserstein loss functional.}
\label{fig:sharedPriorOverview}
\end{figure}

The existence of optimal control parameters for~\eqref{eq:objectiveFunctionSemi} is verified in the next theorem:
\begin{theorem}[Existence of solutions]\label{thm:existenceSolutionSemi}
The minimum in~\eqref{eq:objectiveFunctionSemi} for $\alpha\in[0,1]$ is attained.
\end{theorem}
\begin{proof}
Let $(T^{\supervised,j},T^{\unsupervised,j},\xi^{\supervised,j},\xi^{\unsupervised,j},\theta^j)\in\Gamma$ be a sequence of control parameters such that
\begin{align*}
\overline{x}^j:&\mathcal{Z}^\supervised\to\R^{n_y},&\overline z&\mapsto x(1,\overline z,T^{\supervised,j},\xi^{\supervised,j},\theta^j),\\
\widehat{x}^j:&\mathcal{Z}^\unsupervised\to\R^{n_y},&\widehat z&\mapsto x(1,\widehat z,T^{\unsupervised,j},\xi^{\unsupervised,j},\theta^j)
\end{align*}
are minimizing sequences for~$J$.
Taking into account the smoothness assumptions regarding $\nabla_1\mathcal{D}$ and $\nabla_1\mathcal{R}$ as well as their compact support we can apply the Picard--Lindel\"of Theorem~\cite[Theorem~2.2]{Te12}
and \cite[Theorem~2.17]{Te12} to deduce that the mappings $t\mapsto x(t,\overline z,T^{\supervised,j},\xi^{\supervised,j},\theta^j)$ and $t\mapsto x(t,\widehat z,T^{\unsupervised,j},\xi^{\unsupervised,j},\theta^j)$ are well-defined for $t\in\R$ and differentiable.
Due to the compactness of the control parameter spaces a subsequence of $(T^{\supervised,j},T^{\unsupervised,j},\xi^{\supervised,j},\xi^{\unsupervised,j},\theta^j)$ (not relabeled)
converges to the optimal control parameters $(T^{\supervised,\ast},T^{\unsupervised,\ast},\xi^{\supervised,\ast},\xi^{\unsupervised,\ast},\theta^\ast)\in\Gamma$.
Next, we prove that $\overline{x}^j\to\overline{x}^\ast$ in~$L^1(\mathcal{Z}^\supervised,\Prob^{\supervised,\mathcal{Z}})$ and $\widehat{x}^j\to \widehat{x}^\ast$ in~$L^1(\mathcal{Z}^\unsupervised,\Prob^{\unsupervised,\mathcal{Z}})$, where
$\overline{x}^\ast(\overline z)\coloneqq x(1,\overline z,T^{\supervised,\ast},\xi^{\supervised,\ast},\theta^\ast)$ and $\widehat{x}^\ast(\widehat z)\coloneqq x(1,\widehat z,T^{\unsupervised,\ast},\xi^{\unsupervised,\ast},\theta^\ast)$.
Gronwall's inequality applied to the initial value problem~\cite[Lemma~2.7 \& Theorem~2.8]{Te12} implies that
\begin{align}
\E_{\overline z\sim\training^{\supervised,\mathcal{Z}}}\left(\Vert\overline{x}^\ast(\overline z)-\overline{x}^j(\overline z)\Vert_2\right)&\leq\frac{M^j(\exp(L)-1)}{L},\label{eq:Gronwall1}\\
\E_{\widehat z\sim\training^{\unsupervised,\mathcal{Z}}}\left(\Vert \widehat{x}^\ast(\widehat z)-\widehat{x}^j(\widehat z)\Vert_2\right)&\leq\frac{M^j(\exp(L)-1)}{L}.\label{eq:Gronwall2}
\end{align}
Here, $L$ denotes the Lipschitz constant of the right-hand side of the state equation, i.e.
\[
L\coloneqq\sup_{\substack{(x_1,z_1),(x_2,z_2)\in\R^{n_y}\times\R^{n_z}\\(x_1,z_1)\neq(x_2,z_2)}}\frac{\Vert T^\ast f(x_1,z_1,\xi^\ast,\theta^\ast)-T^\ast f(x_2,z_2,\xi^\ast,\theta^\ast)\Vert_2}{\Vert(x_1,z_1)-(x_2,z_2)\Vert_2},
\]
which is finite due to the compact support of~$\mathcal{D}$ and~$\mathcal{R}$.
Likewise, $M^j$ is given by
\[
M^j\coloneqq\sup_{(x,z)\in\R^{n_y}\times\R^{n_z}}\Vert T^j f(x,z,\xi^j,\theta^j)-T^\ast f(x,z,\xi^\ast,\theta^\ast)\Vert_2,
\]
which is finite following the same line of arguments as above.
Moreover, $M^j\leq C_M(\vert T^\ast-T^j\vert+\Vert\theta^\ast-\theta^j\Vert_2+\Vert\xi^\ast-\xi^j\Vert_2)$ holds true for a constant $C_M>0$.
Thus, \eqref{eq:Gronwall1} and \eqref{eq:Gronwall2} imply that $\overline{x}^j\to\overline{x}^\ast$ in~$L^1(\mathcal{Z}^\supervised,\Prob^{\supervised,\mathcal{Z}})$ and $\widehat{x}^j\to \widehat{x}^\ast$ in~$L^1(\mathcal{Z}^\unsupervised,\Prob^{\unsupervised,\mathcal{Z}})$.
In particular, for both sequences there exist subsequences (not relabeled) that converge pointwise~a.e.
Finally, by taking into account Fatou's lemma we have proven $J(\overline{x}^\ast,\widehat{x}^\ast)\leq\liminf_{j\to\infty}J(\overline x^j,\widehat{x}^j)$, which concludes this proof.
\end{proof}
\begin{remark}
Using a perturbation argument and Gronwall's inequality we can easily prove that the solution satisfies $x\in C^1([0,1],C^0(\R^{n_z}\times[0,\Tmax]\times\Xi\times\Theta,\R^{n_y}))$.
\label{rem:smoothness}
\end{remark}

\section{Discretization of mean-field optimal control problems}\label{sec:discretization}
The aim of this section is the discrete formulation of the mean-field optimal control problems discussed in~\cref{sec:meanfield}.
To this end, we present different discretizations for the gradient flow equation in \cref{sub:discretizationSchemes}.
Furthermore, we elaborate on the discretization of the data fidelity term (\cref{sub:dataFidelity}) and the total deep variation regularizer (\cref{sub:TDV}).
Finally, the discretized mean-field optimal control problems for energy-based learning and shared prior learning are detailed in \cref{sub:discreteSharedPriorLearning}.

\subsection{Discretization schemes for the gradient flow}\label{sub:discretizationSchemes}
In this section, we propose three different discretization schemes for the state equation~\eqref{eq:gradientFlow}.
In all schemes, the number of iteration steps~$S\in\N$ is assumed to be fixed.
Recall that $z\in\R^{n_z}$ refers to the observation, $x_0=A_\init z\in\R^{n_y}$ is an initial value depending on~$z$, $T\in[0,\Tmax]$ is the time horizon, $\xi\in\Xi$ and $\theta\in\Theta$ are the learnable parameters of the data fidelity term and the regularizer, respectively.
The proposed schemes are as follows:
\begin{itemize}
\item
The \emph{explicit forward Euler scheme} as one of the simplest Runge--Kutta schemes~\cite{At89} reads as
\[
x_{s+1}=x_s+f^\explicit(x_s,z,S,T,\xi,\theta)
\]
for $s=0,\ldots,S-1$ with 
\[
f^\explicit(x,z,S,T,\xi,\theta)\coloneqq-\tfrac{T}{S}A^\top\nabla_1\mathcal{D}(Ax,z,\xi)-\tfrac{T}{S}\nabla_1\mathcal{R}(x,\theta).
\]
\item
Next, we introduce the \emph{semi-implicit forward Euler scheme}, in which an implicit update step on the data fidelity term and an explicit step on the regularizer are performed.
However, due to this particular structure we can only apply the scheme for image denoising (i.e.~$A=\Id$) since a closed-form expression for~$x_{s+1}$ is in general not available.
The scheme is given by
\begin{equation}
x_{s+1}=x_s-\tfrac{T}{S}\nabla_1\mathcal{D}(x_{s+1},z,\xi)-\tfrac{T}{S}\nabla_1\mathcal{R}(x_s,\theta)
\label{eq:implicitScheme}
\end{equation}
for $s=0,\ldots,S-1$.
Note that this equation is equivalent to
\[
x_{s+1}=f^\implicit(x_s,z,S,T,\xi,\theta),
\]
where 
\begin{equation}
f^\implicit(x_s,z,S,T,\xi,\theta)\coloneqq(\Id+\tfrac{T}{S}\nabla_1\mathcal{D}(\cdot,z,\xi))^{-1}(x_s-\tfrac{T}{S}\nabla_1\mathcal{R}(x_s,\theta)),
\label{eq:proxData}
\end{equation}
and $(\Id+\tfrac{T}{S}\nabla_1\mathcal{D}(\cdot,z,\xi))^{-1}$ is the proximal operator~\cite{ChPo16,BuSa16}.
\item
The starting point of the \emph{Euler--Newton scheme} is a linearization of $\nabla_1\mathcal{D}$ in~\eqref{eq:implicitScheme} for a general linear operator~$A$ around the base point $x_{s+\frac{1}{2}}=x_s-\frac{T}{S}\nabla_1\mathcal{R}(x_s,\theta)$.
Unlike the semi-implicit discretization, this scheme is applicable for general linear inverse problems and reads as
\[
x_{s+1}=x_{s+\frac{1}{2}}-\tfrac{T}{S}A^\top\nabla_1\mathcal{D}(Ax_{s+\frac{1}{2}},z,\xi)-\tfrac{T}{S}A^\top\nabla_1^2\mathcal{D}(Ax_{s+\frac{1}{2}},z,\xi)(Ax_{s+1}-Ax_{s+\frac{1}{2}})
\]
for $s=0,\ldots,S-1$.
After rearranging the terms we can rewrite this scheme as
\[
x_{s+1}=f^\Newton(x_{s+\frac{1}{2}},z,S,T,\xi,\theta),
\]
where the function $f^\Newton$ is given by
\[
f^\Newton(x,z,S,T,\xi,\theta)\coloneqq x-\left(\tfrac{S}{T}\Id+A^\top\nabla_1^2\mathcal{D}(Ax,z,\xi)A\right)^{-1}A^\top\nabla_1\mathcal{D}(Ax,z,\xi).
\]
Note that this scheme is equivalent to a single Newton step of the proximal problem
\[
\min_{x\in\R^{n_y}}\frac{\Vert x-x_{s+\frac{1}{2}}\Vert_2^2}{2\tfrac{T}{S}}+\mathcal{D}(Ax,z,\xi),
\]
which is initialized with the base point.
We highlight that this scheme is stable due to the proximal structure of $f^\Newton$.
\end{itemize}
We denote by $x^{\kappa,S}:\R^{n_z}\times[0,\Tmax]\times\Xi\times\Theta\to(\R^{n_y})^{S+1}$ the discrete trajectories, which are associated with the discretization scheme $\kappa\in\{\explicit,\implicit,\Newton\}$, the control parameters~$T$, $\xi$ and~$\theta$, and the observation~$z$.
In particular, $x_S^{\kappa,S}(z,T,\xi,\theta)$ refers to the terminal state of this discrete trajectory.

\subsection{Discretization of data fidelity term and proximal operator}\label{sub:dataFidelity}
In this subsection, we elaborate on the discretization of the data fidelity term~$\mathcal{D}$ and the proximal operator discussed in \cref{subsub:variation} and \cref{sub:discretizationSchemes}.
Since the state equation~\eqref{eq:gradientFlow} only requires the gradient of the data fidelity term,
we start with the discretization of $\nabla_1\mathcal{D}$ using a weighted sum of cubic splines in order to fulfill the regularity assumptions.
Since the semi-implicit scheme allows for a reformulation as a proximal operator, we learn a parametric function representing this operator instead of~$\nabla_1\mathcal{D}$ in this case.
Henceforth, we use the notation and the definitions introduced in \cref{subsub:variation}.

In the case of the Fr\'echet metric and the explicit/Euler--Newton scheme, we discretize the data fidelity term as the weighted sum of $2n_\Xi+1$ cubic spline basis functions~$\varphi_j^{n_\Xi}$, $j\in\{-n_\Xi,\ldots,n_\Xi\}$, with $n_\Xi$~degrees of freedom in total due to the symmetry assumption $\rho_F(x,\xi)=\rho_F(-x,\xi)$.
The centers of the basis functions are located at equidistant points on the interval~$[-Q,Q]$ for a fixed $Q>0$.
Each basis function $\varphi_j^{n_\Xi}$ is a translation of a cubic spline reference basis function~$\varphi$ with support $\operatorname{supp}(\varphi)=[-\frac{2Q}{n_\Xi},\frac{2Q}{n_\Xi}]$ and centered around~$\frac{Qj}{n_\Xi}$,
i.e.~$\varphi_j^{n_\Xi}(x)=\varphi(x-\frac{Qj}{n_\Xi})$.
Thus,
\begin{equation}
\nabla_1\rho_F(x,\xi)=\sum_{j=-n_\Xi}^{n_\Xi}\xi_j\varphi_j^{n_\Xi}(x)
\label{eq:discFrechet}
\end{equation}
for $x\in[-Q,Q]$.
In particular, the coefficients of the basis functions satisfy~$\xi_j=-\xi_{-j}$, $\xi_0=0$, $\xi_{j-1}\leq \xi_j$ and $\xi_j\geq 0$ for $j\in\{1,\ldots,n_\Xi\}$.
Hence, $\xi=(\xi_1,\ldots,\xi_{n_\Xi})\in\Xi\coloneqq(\R_0^+)^{n_\Xi}$.

In the case of the generalized divergence and the explicit/Euler--Newton scheme, the domain~$[-Q,Q]^2$ is two-dimensional, that is why we discretize the data fidelity term for $(x,y)\in[-Q,Q]^2$ and $N_\Xi=\frac{\sqrt{n_\Xi}}{2}$ using the same basis functions as before as follows:
\begin{equation}
\nabla_1\rho_\div(x,y,\xi)=\sum_{i=-N_\Xi}^{N_\Xi}\sum_{j=-N_\Xi}^{N_\Xi}\xi_{i,j}\varphi_i^{N_\Xi}(x)\varphi_j^{N_\Xi}(y).
\label{eq:discDivergence}
\end{equation}
The coefficients have to satisfy~$\xi_{i-1,j}\leq\xi_{i,j}$ and $\xi_{i,i}=0$ for $i,j\in\{-N_\Xi,\ldots,N_\Xi\}$, and are adapted such that $\nabla_1\rho_\div(x,x,\xi)=0$ for $x\in[-Q,Q]$.
In summary, $\xi=\{\xi_{i,j}:i,j\in\{-N_\Xi,\ldots,N_\Xi\}\}\in\Xi\coloneqq\R^{n_\Xi}$ with the aforementioned constraints.

Finally, for the semi-implicit scheme we discretize the proximal operator appearing in~\eqref{eq:proxData} using \eqref{eq:discFrechet} or \eqref{eq:discDivergence}, where we additionally impose a $1$-Lipschitz constraint.
Note that the monotonicity is already implied by the constraints on the coefficients.

We remark that in all cases the identity of indiscernibles is ignored for numerical reasons, i.e.~$\mathcal{D}(Ax,z,\xi)=0$ does not necessarily imply $Ax=z$.

\subsection{Total deep variation regularizer}\label{sub:TDV}
In this subsection, we present the convolutional neural network~$\mathcal{N}$ as a vital part of~$\mathcal{R}$, which is taken from the \emph{total deep variation regularizer}~\cite{KoEf20,KoEf20a}.
On the largest scale, the total deep variation is composed of $3$~blocks $\textrm{Bl}^i$, $i=1,2,3$, as depicted in~\Cref{fig:network} (upper left),
where each block is designed as a U-Net~\cite{RoFi15} with $5$~residual blocks $\mathrm{R}_1^i,\ldots,\mathrm{R}_5^i$, as shown in \Cref{fig:network} (lower left).
In particular, residual blocks on the same scale are linked with skip connections (solid vertical arrows).
To enhance the expressiveness of~$\mathcal{N}$, residual connections (dotted vertical arrows) connecting residual blocks on the same scale of consecutive blocks are added whenever possible.
The residual block~$\mathrm{R}_j^i$ for $j=1,\ldots,5$ is modeled as 
\[
\mathrm{R}_j^i(x)=x+K_{j,2}^i\Phi(K_{j,1}^i x)
\]
for $3\times 3$ convolution operators with $m$~feature channels and no bias, which are represented by matrices~$K_{j,1}^i,K_{j,2}^i\in\R^{nm\times nm}$  (\Cref{fig:network}, upper right).
Here, we use the log-Student-t-distribution $\phi(x)=\frac{1}{2}\log(1+x^2)$ as the componentwise activation function of $\Phi=(\phi,\ldots,\phi):\R^{nm}\to\R^{nm}$, which is inspired by
the work of Huang and Mumford~\cite{HuMu99} on the statistics of natural images.
To avoid aliasing, we follow a recent approach by Zhang~\cite{Zh19}, 
who advocated the use of $3\times 3$ convolutions and transposed convolutions with stride~$2$ and a blur kernel to realize downsampling and upsampling.
In total, we use $\vert\Theta\vert\approx 4\cdot 10^5$ learnable parameters.
\begin{figure}
\centering
\includegraphics[width=.8\linewidth]{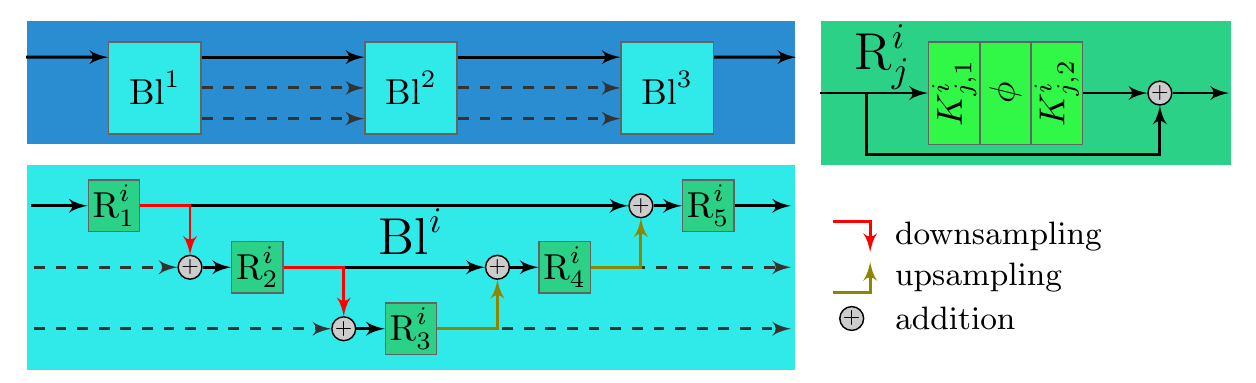}
\caption{The network structure of the total deep variation with $3$~blocks each operating on $3$~scales.}
\label{fig:network}
\end{figure}

\subsection{Discretized mean-field optimal control problems}\label{sub:discreteSharedPriorLearning}
This section is devoted to the introduction of the discretized versions of energy-based learning and shared prior learning.
Intuitively, we solely replace the terminal states of the image trajectories in~\eqref{eq:objectiveFunctionSupervised} and~\eqref{eq:objectiveFunctionSemi} by the corresponding terminal states of the discretized trajectories.

The \emph{discretized mean-field optimal control problem for energy-based learning} is defined as
\begin{equation}
\inf
\left\{\E_{(\overline y,\overline z)\sim\training^\supervised}\loss(x_S^{\kappa,S}(\overline z,T,\xi,\theta),\overline y):T\in[0,\Tmax],\xi\in\Xi,\theta\in\Theta\right\}.
\label{eq:discreteOptimalControlVariationalLearning}
\end{equation}
We recall from \cref{sub:discretizationSchemes} that $x_S^{\kappa,S}(\overline z,T,\xi,\theta)$ denotes the terminal state of the discrete trajectory.
Likewise, the \emph{discretized mean-field optimal control problem for shared prior learning} reads as
\begin{equation}
\inf\left\{J(x_S^{\kappa,S}(\cdot,T^\supervised,\xi^\supervised,\theta),x_S^{\kappa,S}(\cdot,T^\unsupervised,\xi^\unsupervised,\theta)):(T^\supervised,T^\unsupervised,\xi^\supervised,\xi^\unsupervised,\theta)\in\Gamma\right\}
\label{eq:discreteOptimalControlSharedPriorLearning}
\end{equation}
for $\kappa\in\{\explicit,\implicit,\Newton\}$ and~$\alpha\in[0,1]$, where~$J$ is the cost functional defined in~\eqref{eq:costFunctional}.
In particular, \eqref{eq:discreteOptimalControlVariationalLearning} is a special case of~\eqref{eq:discreteOptimalControlSharedPriorLearning} when setting $\alpha=1$.
\begin{theorem}[Existence of minimizers]
The minimizer in~\eqref{eq:discreteOptimalControlSharedPriorLearning} is attained for all discretization schemes $\kappa\in\{\explicit,\implicit,\Newton\}$.
\end{theorem}
\begin{proof}
This proof relies on standard arguments of the direct method in the calculus of variations~\cite{Da08}, that is why we only sketch the proof here.
We note that $x^{\kappa,S}\in C^0(\R^{n_z}\times[0,\Tmax]\times\Xi\times\Theta,(\R^{n_y})^{S+1})$ due to the assumptions regarding~$f$ for all three discretization schemes considered above.
Let $(T^{\supervised,j},T^{\unsupervised,j},\xi^{\supervised,j},\xi^{\unsupervised,j},\theta^j)\in\Gamma$ be a minimizing sequence of control parameters defining a minimizing sequences of states $x^{\kappa,S}(\cdot,T^{\supervised,j},\xi^{\supervised,j},\theta^j)$ and $x^{\kappa,S}(\cdot,T^{\unsupervised,j},\xi^{\unsupervised,j},\theta^j)$.
Due to the compactness of the control parameter space a subsequence thereof converges to $(T^{\supervised,\ast},T^{\unsupervised,\ast},\xi^{\supervised,\ast},\xi^{\unsupervised,\ast},\theta^\ast)\in\Gamma$.
Then, a discretization specific induction argument reveals that the minimizing sequences of states actually converge to 
$x^{\kappa,S}(\cdot,T^{\supervised,\ast},\xi^{\supervised,\ast},\theta^\ast)$ and $x^{\kappa,S}(\cdot,T^{\unsupervised,\ast},\xi^{\unsupervised,\ast},\theta^\ast)$, respectively.
The remainder of this proof is similar to~\Cref{thm:existenceSolutionSemi}.
\end{proof}

\section{Consistency of discretization}\label{sec:Mosco}
In this section, we prove the consistency of the temporal discretization schemes in terms of Mosco convergence~\cite{Mo69} as~$S\to\infty$ (see \Cref{thm:Mosco}).
Based on this result, we can verify the convergence of the associated minimizers in \Cref{thm:MoscoExistence}.

Mosco convergence is a notion of convergence for functionals, which in addition implies the convergence of minimizers under suitable conditions.
We note that Mosco convergence is a stronger version of $\Gamma$-convergence~\cite{Ma93}, since the former implies the latter.
Let us first recall their definitions:
\begin{definition}[Mosco convergence]
Let $(X,d)$ be a metric space.
We consider the functionals $\mathcal{J}^S,\mathcal{J}:X\to\overline\R$ for $S\in\N$.
Then the sequence $\mathcal{J}^S$ converges to $\mathcal{J}$ in the sense of Mosco w.r.t.~the topology induced by $d$ if the following holds:
\begin{enumerate}[label=(M\arabic*),leftmargin=3em]
\item \label{MoscoItem1}
For every sequence $\{x^S\}_{S\in\N} \subset X$ such that $x^S$ converges weakly to $x\in X$ (denoted by $x^S\rightharpoonup x\in X$) the subsequent inequality holds true:
\begin{equation}\tag{liminf-inequality}
\mathcal{J}(x) \leq \liminf_{S\to\infty} \mathcal{J}^S(x^S).
\end{equation}
\item For every $x\in X$ there exists a recovery sequence $\{x^S\}_{S\in\N}\subset X$, i.e.~$x^S\to x$ as $S\to\infty$ and
\begin{equation}\tag{limsup-inequality}
\mathcal{J}(x) \geq \limsup_{S\to\infty} \mathcal{J}^S(x^S).
\end{equation}
\end{enumerate}
If the weak topology in~\ref{MoscoItem1} is replaced by the strong topology induced by $d$,
then $\mathcal{J}^S$ is said to $\Gamma$-converge to $\mathcal{J}$ w.r.t.~the topology induced by $d$.	
\end{definition}
We define $\mathcal{I}^S:C^0(\R^{n_z},\R^{n_y})^{S+1}\to H^1([0,1],C^0(\R^{n_z},\R^{n_y}))\cap C^0([0,1]\times\R^{n_z},\R^{n_y})$ as the affine interpolation in time, i.e.~for any $x\in C^0(\R^{n_z},\R^{n_y})^{S+1}$ we have
\[
\mathcal{I}^S[x](t,z)\coloneqq(s+1-St)x_s(z)+(St-s)x_{s+1}(z)
\]
for $s\in\{0,\ldots,S-1\}$, $t\in[\frac{s}{S},\frac{s+1}{S})$ and $\mathcal{I}^S[x](1,z)\coloneqq x_S(z)$.
Note that $\mathcal{I}^S[x](\frac{s}{S},z)=x_s(z)$.
\begin{remark}
Henceforth, we use the abbreviation $\gamma=(T^\supervised,T^\unsupervised,\xi^\supervised,\xi^\unsupervised,\theta)\in\Gamma$.
The measures of $L^2((0,1)\times\mathcal{Z}^\supervised)$ and $L^2((0,1)\times\mathcal{Z}^\unsupervised)$ are the one-dimensional Lebesgue measure in the first and the respective probability measure in the second component.
For convenience, we set $\mathbb{L}\coloneqq L^2((0,1)\times\mathcal{Z}^\supervised)\times L^2((0,1)\times\mathcal{Z}^\unsupervised)$.
\end{remark}
We define the temporal extension of the cost functional~$J$ in the time-discrete case $\mathcal{J}^S:\mathbb{L}\to\R_0^+$ and in the time-continuous case $\mathcal{J}:\mathbb{L}\to\R_0^+$ as follows:
\begin{align*}
\mathcal{J}^S(\overline X,\widehat X)
&\coloneqq
\begin{cases}
J(\overline x_S,\widehat{x}_S)
&\text{if }\exists\gamma\in\Gamma\text{ with }\overline x(\overline z)=x^{\kappa,S}(\overline z,T^\supervised,\xi^\supervised,\theta)\text{ s.t. }\overline X=\mathcal{I}^S[\overline x],\\
&\widehat{x}(\widehat z)=x^{\kappa,S}(\widehat z,T^\unsupervised,\xi^\unsupervised,\theta)\text{ s.t. }\widehat X=\mathcal{I}^S[\widehat x],\\
+\infty&\text{else},
\end{cases}
\\
\mathcal{J}(\overline X,\widehat X)
&\coloneqq
\begin{cases}
J(\overline X(1,\cdot),\widehat X(1,\cdot))&\text{if }\exists\gamma\in\Gamma\text{ s.t. }\overline X(t,\overline z)=x(t,\overline z,T^\supervised,\xi^\supervised,\theta),\\
&\widehat X(t,\widehat z)=x(t,\widehat z,T^\unsupervised,\xi^\unsupervised,\theta),\\
+\infty&\text{else}.
\end{cases}
\end{align*}
Next, we state the Mosco convergence for the explicit and the Euler--Newton discretization, we omit the proof for the semi-implicit scheme due to its limited applicability for denoising.
\begin{theorem}[Mosco convergence]\label{thm:Mosco}
$\mathcal{J}^S$ converges to~$\mathcal{J}$ in the sense of Mosco w.r.t.~the $\mathbb{L}$-topology as $S\to\infty$ for $\kappa\in\{\explicit,\Newton\}$.
\end{theorem}
The proof is deferred to the end of this section.
Finally, we prove the convergence of minimizers of~$\mathcal{J}^S$ to the respective minimizers of~$\mathcal{J}$ as $S\to\infty$ as well as the convergence of the energies.
\begin{theorem}\label{thm:MoscoExistence}
Let $(\overline{X}^S,\widehat{X}^S)\in\mathbb{L}$ be a sequence of minimizers of~$\mathcal{J}^S$ for $S\in\N$.
Then, a subsequence of $\{(\overline{X}^S,\widehat{X}^S)\}_{S\in\N}$ converges weakly in~$\mathbb{L}$
to a minimizer of~$\mathcal{J}$ and the associated sequence of energies converges to the respective energy.    
\end{theorem}
\begin{proof}
A discretization specific induction argument shows that $(\overline{X}^S,\widehat{X}^S)$ is uniformly bounded in~$\mathbb{L}$ due to 
\[
\E_{\overline z\sim\training^{\supervised,\mathcal{Z}}}(\Vert\overline{X}^S(0,\overline z)\Vert_2^2)\leq \E_{\overline z\sim\training^{\supervised,\mathcal{Z}}}(\Vert A_\init\Vert_2^2\Vert\overline z\Vert_2^2)<\infty
\]
(with an analogous reasoning for~$\widehat{X}^S$) and the compact support of~$f$.
Hence, there exists a subsequence of~$(\overline{X}^S,\widehat{X}^S)$ weakly converging to~$(\overline{X},\widehat{X})\in\mathbb{L}$.
Next, we prove that $(\overline{X},\widehat{X})$ is actually a minimizer of~$\mathcal{J}$.
Following \Cref{thm:existenceSolutionSemi}, there exists a minimizer $(\overline{X}_{\min},\widehat{X}_{\min})\in\mathbb{L}$ of~$\mathcal{J}$.
If $\mathcal{J}(\overline{X}_{\min},\widehat{X}_{\min})<\mathcal{J}(\overline{X},\widehat{X})$, then there exists a recovery $(\overline{X}^S,\widehat{X}^S)\in\mathbb{L}$ satisfying $\limsup_{S\to\infty}\mathcal{J}^S(\overline{X}^S,\widehat{X}^S)\leq\mathcal{J}(\overline{X}_{\min},\widehat{X}_{\min})$.
Hence,
\[
\mathcal{J}(\overline{X},\widehat{X})
\leq\liminf_{S\to\infty}\mathcal{J}^S(\overline{X}^S,\widehat{X}^S)
\leq\limsup_{S\to\infty}\mathcal{J}^S(\overline{X}^S,\widehat{X}^S)
\leq\mathcal{J}(\overline{X}_{\min},\widehat{X}_{\min}),
\]
which contradicts $\mathcal{J}(\overline{X}_{\min},\widehat{X}_{\min})<\mathcal{J}(\overline{X},\widehat{X})$.
In summary, $(\overline{X},\widehat{X})$ is indeed a minimizer of~$\mathcal{J}$ and the associated energies converge.
\end{proof}
Next, we present a sketch of the proof of \Cref{thm:Mosco}.
\begin{proof}
We first note that the pointwise evaluation of~$\overline X$ and~$\widehat X$ appearing in~$\mathcal{J}$ is well-defined due to~\Cref{rem:smoothness}.
In what follows, we prove the $\liminf$--inequality and the $\limsup$--inequality separately.
We remark that this proof is essentially an adaption of the convergence proof for ordinary differential equations.
However, to the best of our knowledge this proof has neither been conducted in the mean-field context nor for the Euler--Newton discretization.

\paragraph{$\liminf$--inequality}
Let $(\overline{X}^S,\widehat{X}^S),(\overline{X},\widehat{X})\in \mathbb{L}$ be sequences such that
$(\overline{X}^S,\widehat{X}^S)\rightharpoonup(\overline{X},\widehat{X})$ in $\mathbb{L}$ as $S\to\infty$.
To exclude trivial cases, we assume that $\mathcal{J}^S(\overline{X}^S,\widehat{X}^S)\leq\overline{\mathcal{J}}$ for a finite constant~$\overline{\mathcal{J}}$ and $\lim_{S\to\infty}\mathcal{J}^S(\overline{X}^S,\widehat{X}^S)=\liminf_{S\to\infty}\mathcal{J}^S(\overline{X}^S,\widehat{X}^S)$.
Hence, there exist $\gamma^S=(T^{\supervised,S},T^{\unsupervised,S},\xi^{\supervised,S},\xi^{\unsupervised,S},\theta^S)\in\Gamma$ such that 
\[
\overline{X}^S=\mathcal{I}^S[x^{\kappa,S}(\overline z,T^{\supervised,S},\xi^{\supervised,S},\theta^S)],\qquad
\widehat{X}^S=\mathcal{I}^S[x^{\kappa,S}(\widehat z,T^{\unsupervised,S},\xi^{\unsupervised,S},\theta^S)].
\]
Further, we can infer that $\gamma^S\to\gamma=(T^\supervised,T^\unsupervised,\xi^\supervised,\xi^\unsupervised,\theta)\in\Gamma$ holds true for a subsequence (not relabeled).
We set $\overline{X}_\infty(t,\overline z)=x(t,\overline z,T^\supervised,\xi^\supervised,\theta)$ and $\widehat{X}_\infty(t,\widehat z)=x(t,\widehat z,T^\unsupervised,\xi^\unsupervised,\theta)$,
and prove that $\overline{X}^S\to\overline{X}_\infty$ in $L^2((0,1)\times\mathcal{Z}^\supervised)$ and $\widehat{X}^S\to\widehat{X}_\infty$ to $L^2((0,1)\times\mathcal{Z}^\unsupervised)$.
Since both proofs are analogous, we only present the former proof and occasionally drop the superscript~$\supervised$.
We define the global error for $s\in\{0,\ldots,S\}$ as
\[
\epsilon_s^S\coloneqq\E_{\overline z\sim\training^{\supervised,\mathcal{Z}}}\left(\Vert \overline{X}^S(\tfrac{s}{S},\overline z)-\overline{X}_\infty(\tfrac{s}{S},\overline z)\Vert_2\right),
\]
and recall for $t_1,t_2\in[0,1]$ that
\begin{equation}    
\E_{\overline z\sim\training^{\supervised,\mathcal{Z}}}\left(\overline{X}_\infty(t_2,\overline z)-\overline{X}_\infty(t_1,\overline z)\right)
=\E_{\overline z\sim\training^{\supervised,\mathcal{Z}}}\left(\int_{t_1}^{t_2}Tf(\overline{X}_\infty(r,\overline z),\overline z,\xi,\theta)\dx r\right).
\label{eq:definitionX}
\end{equation}
Henceforth, we use the notation $\mathcal{X}=\R^{n_y}\times\R^{n_z}\times\Xi\times\Theta$ for the domain space of~$f$.

In the case $\kappa=\explicit$, we immediately see that
\begin{equation}
\overline{X}^S(\tfrac{s+1}{S},\overline z)=\overline{X}^S(\tfrac{s}{S},\overline z)+f^\explicit(\overline{X}^S(\tfrac{s}{S},\overline z),\overline z,S,T^S,\xi^S,\theta^S)
\label{eq:definitionXs}
\end{equation}
for any $s\in\{0,\ldots,S-1\}$ and a.e.~$\overline z\in\mathcal{Z}^\supervised$.
Hence, \eqref{eq:definitionX}, \eqref{eq:definitionXs} and $f^\explicit(x,z,S,T,\xi,\theta)=\frac{T}{S}f(x,z,\xi,\theta)$ imply that
\begin{align*}
\epsilon_{s+1}^S&\leq\epsilon_s^S+\E_{\overline z\sim\training^{\supervised,\mathcal{Z}}}
\left(\int_{\frac{s}{S}}^{\frac{s+1}{S}}\Vert T^S f(\overline{X}^S(\tfrac{s}{S},\overline z),\overline z,\xi^S,\theta^S)-Tf(\overline{X}_\infty(r,\overline z),\overline z,\xi,\theta)\Vert_2\dx r\right)\\
&\leq\epsilon_s^S+\tfrac{1}{S}\vert T^S-T\vert\Vert f\Vert_{C^0(\mathcal{X})}+\tfrac{\Tmax}{2S^2}\vert f\vert_{C^{1,1}(\mathcal{X})}\E_{\overline z\sim\training^{\supervised,\mathcal{Z}}}(\Vert\partial_t\overline{X}_\infty(t,\overline z)\Vert_{L^\infty([0,1])})\\
&\quad+\tfrac{\Tmax}{S}\vert f\vert_{C^{1,1}(\mathcal{X})}\left(\epsilon_s^S+\Vert\xi^S-\xi\Vert_2+\Vert\theta^S-\theta\Vert_2\right).
\end{align*}
The first inequality is implied by Jensen's inequality.
To prove the second inequality, we used the triangle inequality several times as well as the smoothness and compact support of~$f$, and $\int_{\frac{s}{S}}^{\frac{s+1}{S}}(r-\frac{s}{S})\dx r=\frac{1}{2S^2}$.
Note that $\E_{\overline z\sim\training^{\supervised,\mathcal{Z}}}(\Vert\partial_t\overline{X}_\infty(t,\overline z)\Vert_{L^\infty([0,1])})<\infty$ due to the defining equation of~$\overline{X}_\infty$ and the compact support of~$f$.
Then, an induction argument implies that
\begin{align*}
&\epsilon_{s+1}^S\leq(1+\tfrac{\Tmax}{S}\vert f\vert_{C^{1,1}(\mathcal{X})})\epsilon_s^S+\tfrac{\Tmax}{S}\vert f\vert_{C^{1,1}(\mathcal{X})}\left(\Vert\xi^S-\xi\Vert_2+\Vert\theta^S-\theta\Vert_2\right)\\
&\quad+\tfrac{1}{S}\vert T^S-T\vert\Vert f\Vert_{C^0(\mathcal{X})}+\tfrac{\Tmax}{2S^2}\vert f\vert_{C^{1,1}(\mathcal{X})}\E_{\overline z\sim\training^{\supervised,\mathcal{Z}}}(\Vert\partial_t\overline{X}_\infty(t,\overline z)\Vert_{L^\infty([0,1])})\\
&\leq(1+\tfrac{\Tmax}{S}\vert f\vert_{C^{1,1}(\mathcal{X})})^{S+1}\epsilon_0^S+\Bigg(\sum_{i=0}^S\left(1+\tfrac{\Tmax\vert f\vert_{C^{1,1}(\mathcal{X})}}{S}\right)^i\Bigg)\Bigg(\tfrac{1}{S}\vert T^S-T\vert\Vert f\Vert_{C^0(\mathcal{X})}\\
&\quad+\tfrac{\Tmax\vert f\vert_{C^{1,1}(\mathcal{X})}}{S}\left(\Vert\xi^S-\xi\Vert_2+\Vert\theta^S-\theta\Vert_2\right)+\tfrac{\Tmax}{2S^2}\vert f\vert_{C^{1,1}(\mathcal{X})}\E_{\overline z\sim\training^{\supervised,\mathcal{Z}}}(\Vert\partial_t\overline{X}_\infty(t,\overline z)\Vert_{L^\infty([0,1])})\Bigg).
\end{align*}
Thus, the geometric series implies
\begin{align*}
\epsilon_s^S&\leq\left(\exp(\Tmax\vert f\vert_{C^{1,1}(\mathcal{X})})-1\right)
\Big(\Vert\xi^S-\xi\Vert_2+\Vert\theta^S-\theta\Vert_2\\
&\quad+\tfrac{1}{\Tmax\vert f\vert_{C^{1,1}(\mathcal{X})}}\Big(\vert T^S-T\vert\Vert f\Vert_{C^0(\mathcal{X})}+\tfrac{1}{2S}\E_{\overline z\sim\training^{\supervised,\mathcal{Z}}}(\Vert\partial_t\overline{X}_\infty(t,\overline z)\Vert_{L^\infty([0,1])})\Big)\Big).
\end{align*}
In particular, $\lim_{S\to\infty}\max_{s=1,\ldots,S}\epsilon_s^S\to 0$.

In the case $\kappa=\Newton$, we first revisit the defining equation for $\overline{X}^S$:
\begin{equation}
\overline{X}^S(\tfrac{s+1}{S},\overline z)
=\overline{X}^S_\frac{1}{2}(\tfrac{s}{S},\overline z)-(\tfrac{S}{T}\Id+A^\top\nabla_1^2\mathcal{D}(A\overline{X}^S_\frac{1}{2}(\tfrac{s}{S},\overline z),z,\xi^S)A)^{-1}
A^\top\nabla_1\mathcal{D}(A\overline{X}^S_\frac{1}{2}(\tfrac{s}{S},\overline z),z,\xi^S),
\label{eq:definitionXsEN}
\end{equation}
where $\overline{X}^S_\frac{1}{2}(\tfrac{s}{S},\overline z)=\overline{X}^S(\tfrac{s}{S},\overline z)-\frac{T^S}{S}\nabla_1\mathcal{R}(\overline{X}^S(\tfrac{s}{S},\overline z),\theta^S)$.
A Taylor expansion of~$\nabla_1\mathcal{D}$ at the base point $\overline{X}^S_\frac{1}{2}(\tfrac{s}{S},\overline z)$ yields
\begin{align}
\nabla_1\mathcal{D}(A\overline{X}^S(\tfrac{s}{S},\overline z),z,\xi^S)
&=\nabla_1\mathcal{D}(A\overline{X}^S_\frac{1}{2}(\tfrac{s}{S},\overline z),z,\xi^S)\notag\\
&\quad+\tfrac{T^S}{S}\nabla_1^2\mathcal{D}(A\overline{X}^S_\frac{1}{2}(\tfrac{s}{S},\overline z),z,\xi^S)A\nabla_1\mathcal{R}(\overline{X}^S(\tfrac{s}{S},\overline z),\theta^S)
+\mathcal{O}(S^{-2}).
\label{eq:TaylorBasepoint}
\end{align}
Combining \eqref{eq:definitionX}, \eqref{eq:definitionXsEN} and \eqref{eq:TaylorBasepoint} we get
\begin{align}
\epsilon_{s+1}^S&\leq\epsilon_s^S+\E_{\overline z\sim\training^{\supervised,\mathcal{Z}}}
\bigg(\int_{\frac{s}{S}}^{\frac{s+1}{S}}\Vert
Tf(\overline{X}_\infty(r,\overline z),\overline z,\xi,\theta)+T^S\nabla_1\mathcal{R}(\overline{X}^S(\tfrac{s}{S},\overline z),\theta^S)\notag\\
&\quad+T^S(\Id+\tfrac{T^S}{S}A^\top\nabla_1^2\mathcal{D}(A\overline{X}^S_\frac{1}{2}(\tfrac{s}{S},\overline z),\overline z,\xi^S)A)^{-1}
A^\top\nabla_1\mathcal{D}(A\overline{X}^S_\frac{1}{2}(\tfrac{s}{S},\overline z),\overline z,\xi^S)
\Vert_2\dx r\bigg)\notag\\
&\leq\epsilon_s^S+\E_{\overline z\sim\training^{\supervised,\mathcal{Z}}}
\bigg(\int_{\frac{s}{S}}^{\frac{s+1}{S}}\Vert
Tf(\overline{X}_\infty(r,\overline z),\overline z,\xi,\theta)\notag\\
&\quad-T^S(\Id+\tfrac{T^S}{S}A^\top\nabla_1^2\mathcal{D}(A\overline{X}^S_\frac{1}{2}(\tfrac{s}{S},\overline z),\overline z,\xi^S)A)^{-1}
f(\overline{X}^S(\tfrac{s}{S},\overline z),\overline z,\xi^S,\theta^S)
\Vert_2\dx r\bigg)\!+\!\mathcal{O}(S^{-3}).
\label{eq:consistencyEN}
\end{align}
Let $\widehat \sigma=\sup_{\widehat x\in\R^{n_y},\widehat z\in\R^{n_z},\widehat \xi\in\Xi}\sigma_{\max}(A^\top\nabla_1^2\mathcal{D}(\widehat x,\widehat z,\widehat \xi)A)$ be the absolute value of the largest eigenvalue of the matrix~$A^\top\nabla_1^2\mathcal{D}(\cdot)A$,
which implies $\Vert\Id-(\Id+\frac{T^S}{S}A^\top\nabla_1^2\mathcal{D}(\cdot)A)^{-1}\Vert\leq\vert1-\frac{1}{1+\frac{\Tmax}{S}\widehat\sigma}\vert$.
Then, in combination with \eqref{eq:consistencyEN} we obtain by neglecting the higher order terms
\begin{align*}
\epsilon_{s+1}^S&\leq\epsilon_s^S+\tfrac{1}{S}\vert T^S-T\vert\Vert f\Vert_{C^0(\mathcal{X})}+\tfrac{\Tmax}{2S^2}\vert f\vert_{C^{1,1}(\mathcal{X})}\E_{\overline z\sim\training^{\supervised,\mathcal{Z}}}(\Vert\partial_t\overline{X}_\infty(t,\overline z)\Vert_{C^0([0,1])})\\
&\quad+\tfrac{\Tmax}{S}(\vert f\vert_{C^{1,1}(\mathcal{X})}(\epsilon_s^S+\Vert\xi^S-\xi\Vert_2+\Vert\theta^S-\theta\Vert_2)+\vert1-\tfrac{1}{1+\frac{\Tmax}{S}\widehat\sigma}\vert\Vert f\Vert_{C^0(\mathcal{X})}).
\end{align*}
As above, we can deduce using an induction argument that
\begin{align*}
\epsilon_s^S&\leq\left(\exp(\Tmax\vert f\vert_{C^{1,1}(\mathcal{X})})-1\right)
\Big(\Vert\xi^S-\xi\Vert_2+\Vert\theta^S-\theta\Vert_2+\tfrac{1}{2S}\E_{\overline z\sim\training^{\supervised,\mathcal{Z}}}(\Vert\partial_t\overline{X}_\infty(t,\overline z)\Vert_{C^0([0,1])})\\
&\quad+\tfrac{1}{\Tmax\vert f\vert_{C^{1,1}(\mathcal{X})}}(\vert T^S-T\vert\Vert f\Vert_{C^0(\mathcal{X})}+\vert1-\tfrac{1}{1+\frac{\Tmax}{S}\widehat\sigma}\vert\Vert f\Vert_{C^0(\mathcal{X})}))\Big).
\end{align*}
Again, $\lim_{S\to\infty}\max_{s=1,\ldots,S}\epsilon_s^S\to 0$.

Thus, for both cases we can estimate as follows:
\begin{align}
&\E_{\overline z\sim\training^{\supervised,\mathcal{Z}}}\!
\left(\int_0^1\Vert \overline{X}^S(r,\overline z)-\overline{X}_\infty(r,\overline z)\Vert_2^2\dx r\right)
\!=\!
\E_{\overline z\sim\training^{\supervised,\mathcal{Z}}}\!
\left(\sum_{s=0}^{S-1}\int_\frac{s}{S}^\frac{s+1}{S}\Vert \overline{X}^S(r,\overline z)-\overline{X}_\infty(r,\overline z)\Vert_2^2\dx r\right)\notag\\
\leq&
\E_{\overline z\sim\training^{\supervised,\mathcal{Z}}}\!
\Bigg(\sum_{s=0}^{S-1}\int_\frac{s}{S}^\frac{s+1}{S}\!\!\!\!\Vert 
(s+1-Sr)(\overline{X}^S(\tfrac{s}{S},\overline z)-\overline{X}_\infty(\tfrac{s}{S},\overline z))+(Sr-s)(\overline{X}_\infty(\tfrac{s+1}{S},\overline z)-\overline{X}_\infty(r,\overline z))\notag\\
&\hspace{3em}+(s+1-Sr)(\overline{X}_\infty(\tfrac{s}{S},\overline z)-\overline{X}_\infty(r,\overline z))+(Sr-s)(\overline{X}^S(\tfrac{s+1}{S},\overline z)-\overline{X}_\infty(\tfrac{s+1}{S},\overline z))
\Vert_2^2\dx r\Bigg).
\label{eq:globalEstimateX}
\end{align}
Let $\delta>0$ and $K_\delta\subset\mathcal{Z}^\supervised$ be compact such that $\E_{\overline z\sim\training^{\supervised,\mathcal{Z}}}(\Vert\overline z\Vert_2^2\mathbb{I}_{\mathcal{Z}^\supervised\setminus K_\delta}(\overline z))<\delta$.
Thus, due to $\overline{X}^S(0,\overline z)=\overline{X}_\infty(0,\overline z)=A_\init\overline z$ and the compact support of~$f$, $f^\explicit$ and $f^\Newton$
there exists a finite constant~$C_\delta$ and a modulus of continuity~$\omega:\R_0^+\to\R_0^+$ with $\lim_{\delta\searrow 0}\omega(\delta)=0$ such that
\begin{align*}
\sup_{S\in\N}\max_{t\in[0,1]}\max_{\overline z\in K_\delta}\Vert \overline{X}^S(t,\overline z)-\overline{X}_\infty(t,\overline z)\Vert_2=C_\delta&<\infty,\\
\sup_{S\in\N}\max_{t\in[0,1]}\E_{\overline z\sim\training^{\supervised,\mathcal{Z}}}\left(\mathbb{I}_{\mathcal{Z}^\supervised\setminus K_\delta}(\overline z)\Vert\overline{X}^S(t,\overline z)-\overline{X}_\infty(t,\overline z)\Vert_2\right)&<\omega(\delta).
\end{align*}
Hence, the first summand of~\eqref{eq:globalEstimateX} can be bounded from above as follows:
\begin{align*}
\E_{\overline z\sim\training^{\supervised,\mathcal{Z}}}
\left(\sum_{s=0}^{S-1}\int_\frac{s}{S}^\frac{s+1}{S}\Vert 
(s+1-Sr)(\overline{X}^S(\tfrac{s}{S},\overline z)-\overline{X}_\infty(\tfrac{s}{S},\overline z))\Vert_2^2\dx r\right)
\leq \omega(\delta)^2+C_\delta\max_{s=1,\ldots,S}\epsilon_s^S,
\end{align*}
where we used $\vert s+1-Sr\vert\leq 1$ for $r\in(\frac{s}{S},\frac{s+1}{S})$.
To estimate the third summand of~\eqref{eq:globalEstimateX}, we first note that a finite constant~$C_f$ exists such that
\[
\max_{t\in[0,1]}\E_{\overline z\sim\training^{\supervised,\mathcal{Z}}}\left(\Vert\partial_t\overline{X}_\infty(t,\overline z)\Vert_2\right)\leq\sup_{(x,\overline z)\in\R^{n_y}\times\mathcal{Z}^\supervised}\Vert Tf(x,\overline z,\xi,\theta)\Vert_2\eqqcolon C_f<\infty.
\]
Hence, a Taylor expansion implies that $\sup_{\overline z\in\mathcal{Z}^\supervised}\Vert\overline{X}_\infty(\tfrac{s}{S},\overline z)-\overline{X}_\infty(r,\overline z)\Vert_2\leq C_f(r-\tfrac{s}{S})$
for any $r\in[\frac{s}{S},\frac{s+1}{S}]$ and $s=0,\ldots,S-1$.
In summary, we can estimate as follows:
\begin{align*}
\E_{\overline z\sim\training^{\supervised,\mathcal{Z}}}
\left(\sum_{s=0}^{S-1}\int_\frac{s}{S}^\frac{s+1}{S}\Vert 
(s+1-Sr)(\overline{X}_\infty(\tfrac{s}{S},\overline z)-\overline{X}_\infty(r,\overline z))
\Vert_2^2\dx r\right)
\leq\tfrac{C_f^2}{30S^2}.
\end{align*}
Since the remaining summands in~\eqref{eq:globalEstimateX} can be estimated analogously, we have shown the convergence of $\overline{X}^S\to\overline{X}_\infty$ in $L^2((0,1)\times\mathcal{Z}^\supervised)$.
In particular, $\overline X=\overline{X}_\infty$ and $\widehat X=\widehat{X}_\infty$ almost everywhere and we have proven the $\liminf$--inequality.

\paragraph{$\limsup$--inequality}
Let $(\overline{X},\widehat{X})\in \mathbb{L}$ be such that $\mathcal{J}(\overline{X},\widehat{X})<\infty$.
Then, there exists $(T^\supervised,T^\unsupervised,\xi^\supervised,\xi^\unsupervised,\theta)\in\Gamma$ such that
\[
\overline{X}(t,\overline z)=x(t,\overline z,T^\supervised,\xi^\supervised,\theta),\qquad
\widehat{X}(t,\widehat z)=x(t,\widehat z,T^\unsupervised,\xi^\unsupervised,\theta).
\]
We define the recovery sequence as follows:
\[
\overline{X}^S=\mathcal{I}^S[x^{\kappa,S}(\overline z,T^\supervised,\xi^\supervised,\theta)],\qquad
\widehat{X}^S=\mathcal{I}^S[x^{\kappa,S}(\widehat z,T^\unsupervised,\xi^\unsupervised,\theta)]
\]
for $\overline z\in\mathcal{Z}^\supervised$ and $\widehat z\in\mathcal{Z}^\unsupervised$.
Note that $\mathcal{J}^S(\overline{X}^S,\widehat{X}^S)<\infty$.

Then, the proof of the $\limsup$--inequality essentially follows the same line of arguments as in the $\liminf$--inequality, we skip further details.
\end{proof}

\section{Numerical results}\label{sec:numerics}
In this section, we present numerical experiments for energy-based learning (\cref{sub:numericsVariationalLearning}) and shared prior learning (\cref{sub:numericsSharedPriorLearning}).

\paragraph{Discretization of gradient flow}
Depending on the task, we use the squared $\ell^2$-norm~$\mathcal{D}_{\ell^2}$, the Fr\'echet metric~$\mathcal{D}_{F}$ with $n_\Xi=256$, or the generalized divergence~$\mathcal{D}_{\div}$ with $N_\Xi=256$ as the data fidelity term,
where in both latter cases the data fidelity term is initialized as the squared $\ell^2$-norm (see \cref{sub:dataFidelity}).
In all experiments, we use the TDV regularizer introduced in~\cref{sub:TDV}.
The gradient flows are discretized with $S=10$ iteration steps.
For the Euler--Newton scheme, $10$~iterations of the conjugate gradient method are applied to approximately solve the linear system appearing in~$f^\Newton$.
Note that the semi-implicit scheme is only applied to image denoising.

\paragraph{General setting for training}
During training, we optimize the respective discretized mean-field optimal control problems w.r.t.~the control parameters~$\gamma$ for different discretization schemes~$\kappa$
using ADAM~\cite{KiBa15} with momentum variables $\beta_1=0.5$ and $\beta_2=0.9$.
Here, we apply random rotations by multiples of~$90^\circ$ and random flips for data augmentation.
Henceforth, to describe a certain setting, we frequently use the shorthand notation ($\ell^p$-cost/$f^\kappa$/$\mathcal{D}$) for $p\in\{1,2\}$, $\kappa\in\{\explicit,\implicit,\Newton\}$ and $\mathcal{D}\in\{\mathcal{D}_{\ell^2},\mathcal{D}_F,\mathcal{D}_\div\}$,
where the first component defines the loss functional~$\loss$ in the supervised subproblem with regularization parameter~$\iota=10^{-3}$ if $p=1$ (see \cref{subsub:supervised}), the second component refers to the considered discretization scheme of the gradient flow incorporating the data fidelity term specified in the last component.
In all experiments, we use $p=1$ in the cost functional for the Wasserstein distance (see \cref{sub:Wasserstein}).
Unless otherwise specified, the observation-generating function is always $Z(Ay,\zeta)=Ay+\zeta$.

\subsection{Energy-based learning}\label{sub:numericsVariationalLearning}
We present numerical results for energy-based learning applied to image denoising, demosaicing and single image super-resolution (SISR), for which we achieve state-of-the-art performance in many cases.

\paragraph{Training and data set}
In all experiments, we compute all control parameters of the discretized optimal control problem~\eqref{eq:discreteOptimalControlVariationalLearning}, i.e. the stopping time~$T$ as well as the parameters of the learned data fidelity term~$\xi$ and the learned regularizer~$\theta$, using a batch size of~$32$ and a patch size of $99\times 99$ with $10^6$~iterations of the ADAM optimizer.
The only difference in the optimization among the applications stems from the initial learning rate, which is $10^{-4}$ in the case of SISR and Poisson denoising,
$4\cdot 10^{-4}$ for salt-and-pepper noise, and $4\cdot 10^{-3}$ in all remaining cases.
All learning rates are divided by~$4$ after $50000$~iterations.
The training was performed on the BSDS400 data set~\cite{MaFo01} (all images have a resolution of $n_y=481\cdot 321=154\,401$ pixels) for denoising and SISR, and on the MSR data set~\cite{KhNo14} for demosaicing.
All images are scaled to the interval~$[0,1]$ before processing.
To account for boundary artifacts in the case of Laplace and salt-and-pepper noise, we incorporate reflective padding with $10$~pixels on each side.
We emphasize that no further task-specific adaptations are required, which experimentally validates that our method is broadly applicable to a wide range of tasks.

\subsubsection{Image denoising}
In this subsection, we present numerical results for the image denoising task with $A=A_\init=\Id$.

\paragraph{Additive white Gaussian noise}
As a first example, we consider \emph{additive white Gaussian noise} with $\zeta\sim\mathcal{N}(0,\sigma^2)$ for~$\sigma\in\{15,25\}$ (see \eqref{eq:inverseProblem}).

First, we empirically verify that the learned data fidelity terms approximate the statistically justified squared $\ell^2$-norm in the case of additive white Gaussian denoising with a Gaussian prior~\cite{Ni07,AsBo13}.
\Cref{fig:proxMapsDenoisingGaussian} visualizes the learned data fidelity terms for $\mathcal{D}_F$ (first pair) and $\mathcal{D}_\div$ (second pair).
In detail, the first and third columns contain color-coded contour plots of both data fidelity terms in the $Ax$-$z$-plane along with three distinct slice functions in the second/fourth column for $z\in\{0,0.5,1\}$.
As a result, both data fidelity terms roughly coincide with the squared $\ell^2$-norm.
However, $\mathcal{D}_F$ is smoother due to the restrictions inherited from the Fr\'echet metric compared to the oscillatory behavior of~$\mathcal{D}_\div$.
Moreover, significantly higher slopes of the data fidelity term can be observed for image intensities outside of the range $[0,1]$.
\begin{figure}[htb]
\includegraphics[width=\linewidth]{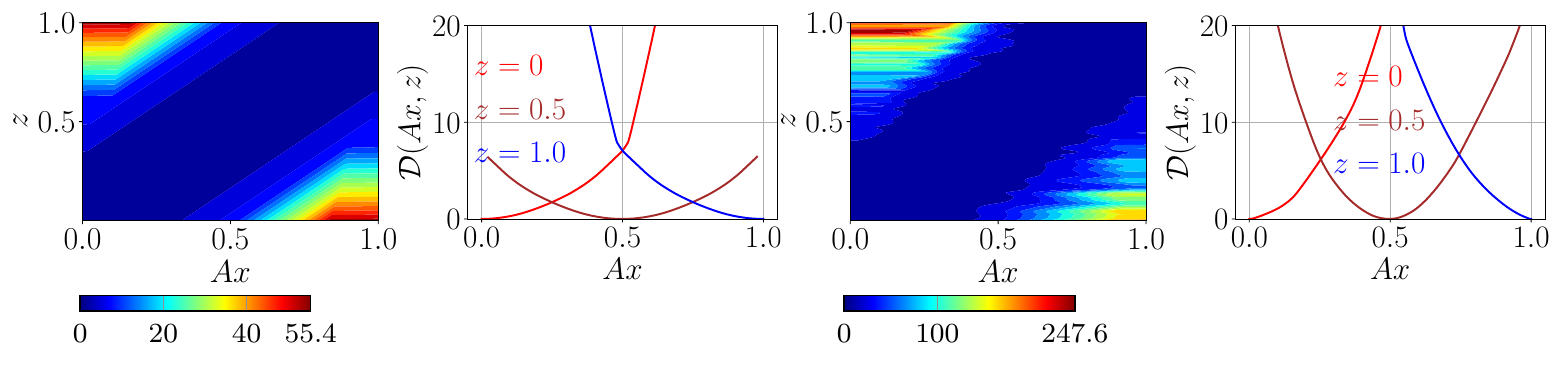}
\caption{
Visualization of contour plots of the learned data fidelity terms along with corresponding slice functions for varying~$z$ for additive white Gaussian noise ($\sigma=25$) using ($\ell^2$-cost/$f^\implicit$/$\mathcal{D}_F$) (first pair) and ($\ell^2$-cost/$f^\implicit$/$\mathcal{D}_\div$) (second pair).}
\label{fig:proxMapsDenoisingGaussian}
\end{figure}

\Cref{tab:PSNRDenoisingSVGaussian} lists the PSNR values for additive white Gaussian denoising ($\sigma\in\{15,25\}$) for several discretization schemes.
We highlight that our proposed method yields PSNR values that are on par with the state-of-the-art method FOCNet, which has $53\,513\,120$ parameters and thus more than $100$ times the number of learnable parameters of our approach.
\begin{table}[htb]
\begin{subtable}[ht]{\textwidth}
\centering
\begin{tabular}{*{8}{c|} }
noisy &  $\mathcal{D}_{\ell^2}$ & $\mathcal{D}_F$ & $\mathcal{D}_\div$ & BM3D~\cite{DaFo07} & TNRD~\cite{ChPo17} & DnCNN~\cite{ZhZu17} & FOCNet~\cite{JiLi19} \\\hline
$24.61$ & $31.82$ & $31.82$ & $31.82$ & $31.08$ & $31.42$ & $31.73$ & $\mathbf{31.83}$
\end{tabular}
\caption{Additive white Gaussian noise $\sigma=15$ using ($\ell^2$-cost/$f^\implicit$).} 
\end{subtable}

\begin{subtable}[ht]{\textwidth}
\resizebox{\linewidth}{!}{
\begin{tabular}{ c || *{13}{c|} }
&\multirow{2}{*}{noisy}&\multicolumn{2}{ |c| }{$\mathcal{D}_{\ell^2}$} & \multicolumn{3}{ |c| }{$\mathcal{D}_F$} &  \multicolumn{3}{ |c|}{$\mathcal{D}_\div$} & BM3D & TNRD  & DnCNN  & FOCNet  \\\cline{3-10}
&& $f^\explicit$ & $f^\implicit$ & $f^\explicit$ & $f^\implicit$ & $f^\Newton$ & $f^\explicit$ & $f^\implicit$ & $f^\Newton$ & \cite{DaFo07} & \cite{ChPo17} & \cite{ZhZu17} & \cite{JiLi19}\\ \hline
$\ell_\iota^1$-cost & \multirow{2}{*}{20.17} & $29.31$ &  $29.31$ & $29.32$ & $29.32$ & $29.32$ & $29.26$ & $29.26$ & $29.28$ & \multirow{2}{*}{$28.57$} & \multirow{2}{*}{$28.92$} & \multirow{2}{*}{$29.23$} & \multirow{2}{*}{$\mathbf{29.38}$}\\
$\ell^2$-cost & & $29.32$ & $29.35$ & $29.34$ & $29.33$ & $29.34$ & $29.27$ & $29.30$ & $29.36$ & & & &
\end{tabular}
}
\caption{Additive white Gaussian noise $\sigma=25$ using various discretization schemes.}
\end{subtable}
\caption{PSNR values for additive Gaussian denoising with various discretization schemes.}
\label{tab:PSNRDenoisingSVGaussian}
\end{table}

Finally, \Cref{fig:GaussianResults} depicts the ground truth image along with the noisy images deteriorated by Gaussian noise ($\sigma=25$), and the restored output image using ($\ell^2$-cost/$f^\implicit$/$\mathcal{D}_\div$).
Our proposed method is clearly capable of removing noise patterns and achieving superior quality, where piecewise smooth regions without fine details are preferred.

\begin{figure}[htb]
\includegraphics[width=\linewidth]{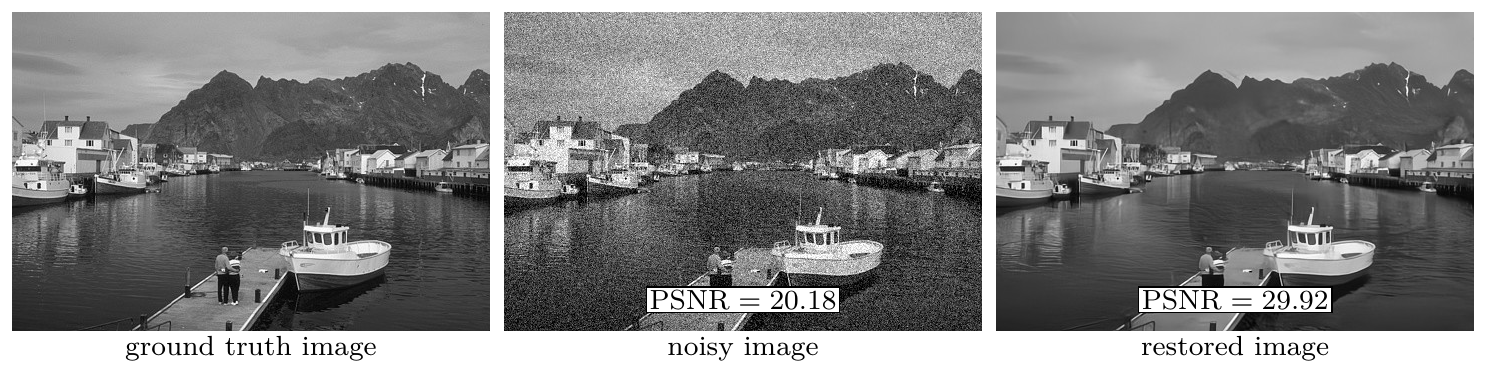}
\caption{From left to right: ground truth image, image corrupted by additive Gaussian noise ($\sigma=25$), restored image using ($\ell^2$-cost/$f^\implicit$/$\mathcal{D}_\div$).}
\label{fig:GaussianResults}
\end{figure}

\paragraph{Further noise distributions}
In what follows, we pursue further numerical experiments with the subsequent different noise instances to illustrate the broad applicability of our method,
where we particularly focus on the resulting learned data fidelity terms:
\begin{itemize}
\item
\emph{additive mixture noise}, where $10\%$ of the pixels are corrupted by additive uniform noise in the given range~$[-25,25]$, 
$20\%$ of the pixels are deteriorated by additive Gaussian noise drawn from~$\mathcal{N}(0,\Id)$, all remaining pixels are altered by additive Gaussian noise drawn from~$\mathcal{N}(0,0.1\cdot\Id)$,
\item
\emph{additive Laplace noise} with standard deviation~$\sigma=25$,
\item
\emph{multiplicative Poisson noise} with peak value~$4$,
\item
in \emph{salt-and-pepper noise}, as a specific instance of impulse noise, $50\%$ of the pixels are corrupted.
\end{itemize}
For Poisson noise, we set $Z(y,\zeta)=y\odot\zeta$ with $\odot$ denoting elementwise multiplication.
In the case of salt-and-pepper noise, $Z$ assigns the values~$0$ or~$1$ (with equal probability) to $50\%$ of the pixels depending on~$\zeta$, all other pixels remain unchanged.
First, we visualize the learned data fidelity terms for mixture and Poisson noise in~\Cref{fig:proxMapsDenoisingMixedPoisson} with the same arrangement as in \Cref{fig:proxMapsDenoisingGaussian}.
As a result, the data fidelity term~$\mathcal{D}_F$ associated with mixed noise is visually much smoother compared to~$\mathcal{D}_\div$ due to the structural assumptions imposed by the Fr\'echet metric, whereas in the
case of the generalized divergence data points away from the diagonal are less likely, which results in the oscillatory behavior.
In the case of multiplicative Poisson noise, $\mathcal{D}_F$ and $\mathcal{D}_\div$ significantly differ, where again $\mathcal{D}_F$ appears to be much smoother.
Furthermore, we note that the Fr\'echet metric is not capable of properly reflecting the statistical distribution as it only depends on~$Ax-z$ and not on the actual value of~$z$.

\Cref{fig:proxMapsDenoising} contains the corresponding data fidelity terms along with additional plots of the proximal maps for salt-and-pepper noise and their slice functions, where both data fidelity terms significantly differ.
Note that the dark red color indicates the value~$+\infty$ for the data fidelity term in the case~$\mathcal{D}_\div$, where we stress that for $z=0$ and $z=1$ finite values are observed, yielding a $Z$-shaped domain.
Thus, the generalized divergence enforces nearly exact data consistency for all pixel values strictly between~$0$ and~$1$ and exhibits a slight bias towards the boundary for~$0$ and~$1$, which cannot be reflected by the Fr\'echet metric.

\begin{figure}[htb]
\includegraphics[width=\linewidth]{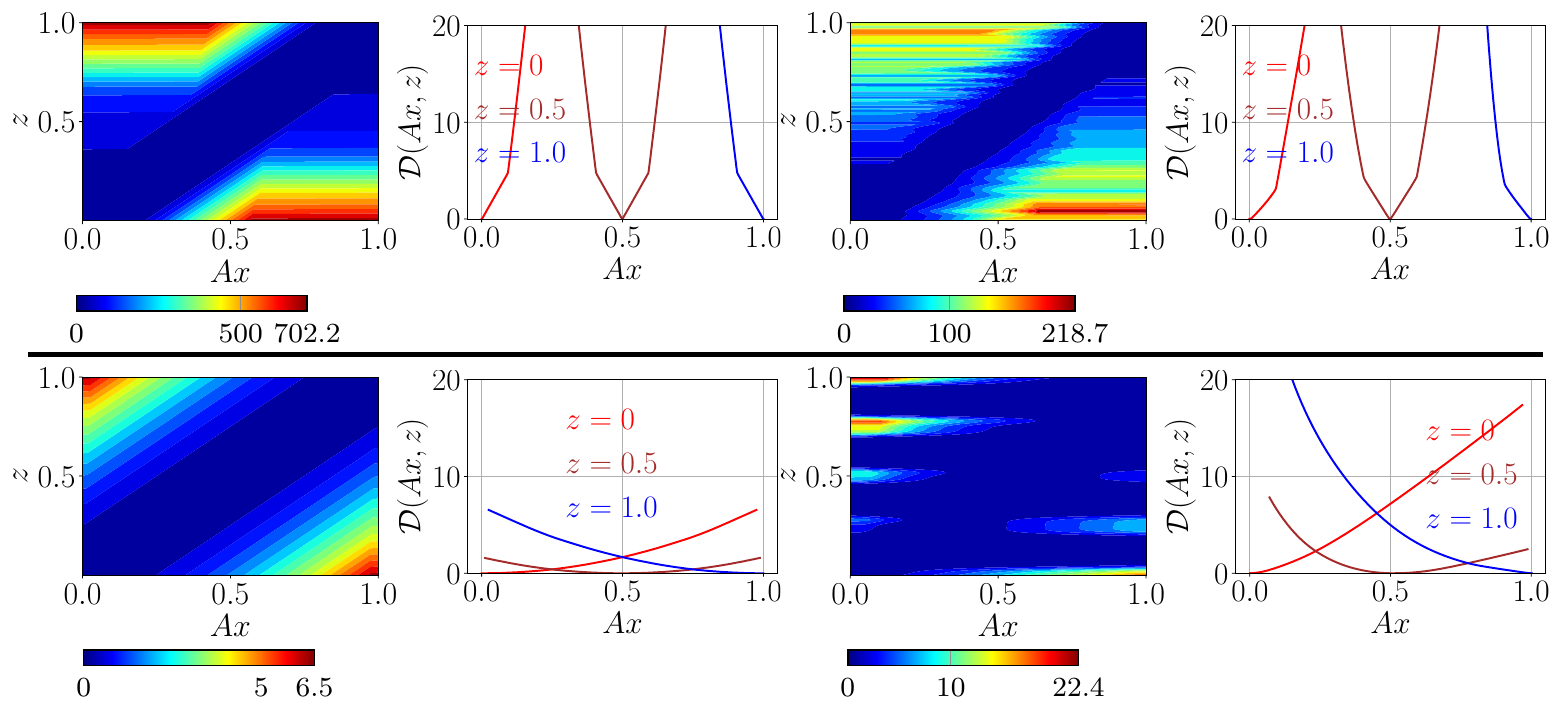}
\caption{
Learned data fidelity terms for additive mixture noise (first row, ($\ell^2$-cost/$f^\implicit$)) and multiplicative Poisson (second row, ($\ell^2$-cost/$f^\implicit$))
with pairs of data fidelity terms and corresponding slice functions for $\mathcal{D}_F$ (first/second column) and $\mathcal{D}_\div$ (third/fourth column).}
\label{fig:proxMapsDenoisingMixedPoisson}
\end{figure}

\begin{figure}[htb]
\includegraphics[width=\linewidth]{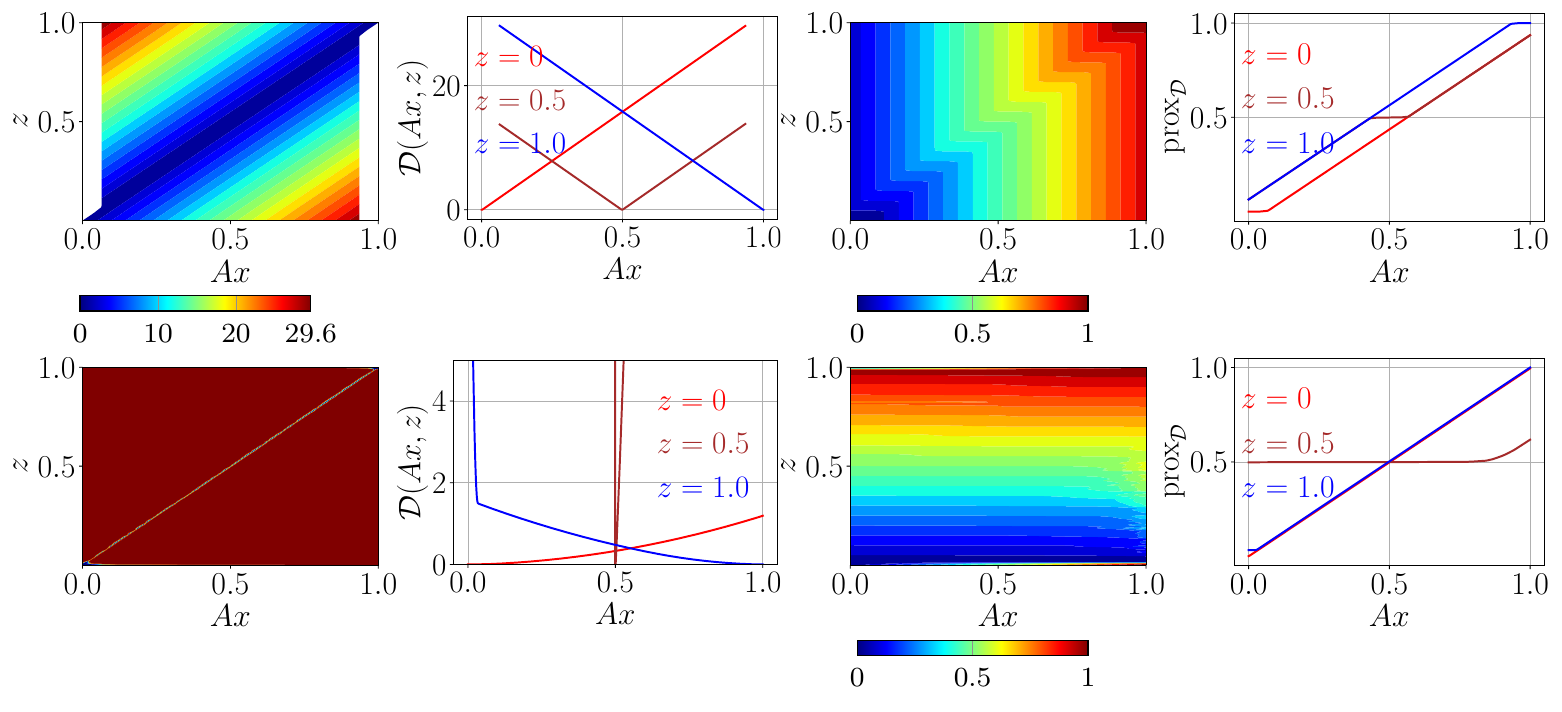}
\caption{
Visualization of learned data fidelity terms for salt-and-pepper noise using ($\ell^2$-cost/$f^\implicit$/$\mathcal{D}_F$) (first row) and ($\ell^2$-cost/$f^\implicit$/$\mathcal{D}_\div$) (second row).
From left to right: contour plots of the learned data fidelity terms (first column) along with corresponding slice functions for varying~$z$ (second column), contour plots of the learned proximal operators (third column) and corresponding slice functions (fourth column).
Note that dark red pixels in the leftmost plot in the second row indicate the value~$\infty$.}
\label{fig:proxMapsDenoising}
\end{figure}

\Cref{tab:PSNRDenoisingSV} lists the PSNR values of some denoising tasks for distinct discretization schemes.
In all tasks, the best discretization schemes of our approach outperform competing methods, where some methods are particularly designed for a distinct noise distribution.
The PSNR values for mixture noise and salt-and-pepper noise significantly vary depending on the scheme, whereas for Laplace and Poisson noise the PSNR values do not differ much, which is reflected by the learned data fidelity terms.
In some cases such as salt-and-pepper and additive Laplace denoising, no competing benchmarks are available for the same data set, that is why we added benchmark results for DnCNN-B~\cite{ZhZu17}, BM3D~\cite{DaFo07}, and TV denoising~\cite{ChLi97},
where the optimal balance parameters are computed by a grid search with an exactness of~$10^{-3}$.
In summary, this experiment validates the necessity to learn the data fidelity term and proves the broad applicability of our proposed method.

\Cref{fig:MixedResults} depicts the ground truth along with the noisy images deteriorated by mixed noise.
Since all discretization schemes yield visually similar results, we only present the restored output image using ($\ell^2$-cost/$f^\implicit$/$\mathcal{D}_\div$).
In the zoom, it is clearly observable that the ground truth and the restored image are nearly indistinguishable.
\Cref{fig:LaplacePoissonResults} shows sequences of ground truth, noisy, and restored images computed with three different discretization schemes for additive Laplace (first row) and multiplicative Poisson noise (second row).
Note that the restored images originally corrupted by additive Laplace noise using $\mathcal{D}_{\ell^2}$ exhibit artifacts in the case of large outliers, which can be seen, for instance, on the border of the vase.
In contrast, $\mathcal{D}_F$ or $\mathcal{D}_\div$ are tuned for Laplace noise and are therefore able to deal with outliers more easily.
\Cref{fig:SNPResults} shows the ground truth image, the corresponding noisy image corrupted by salt-and-pepper noise as well as all proposed discretization schemes using the $\ell^2$-cost functional.
Here, the results computed with ($f^\implicit$/$\mathcal{D}_\div$) and ($f^\Newton$/$\mathcal{D}_\div$) significantly outperform all remaining results.
Besides, the results computed with~$\mathcal{D}_{\ell^2}$ exhibit clearly visible noise patterns, which can, for instance, be seen in the face region.
Moreover, the explicit discretization yields visually and quantitatively inferior results compared to the remaining discretization schemes, which justifies the inclusion of the Euler--Newton scheme as a replacement for the semi-implicit scheme for general imaging tasks.

To sum up, in all cases the restored images are visually and in terms of the PSNR value significantly better than the noisy input images, the gain in the PSNR value is noise-dependent and roughly in the range between~$4$ and~$20$.

\begin{table}[htb]
\centering
\begin{subtable}[ht]{\textwidth}
\centering
\begin{tabular}{*{6}{c|} }
noisy &  $\mathcal{D}_{\ell^2}$ & $\mathcal{D}_F$ & $\mathcal{D}_\div$ & BM3D~\cite{DaFo07} & GCBD~\cite{Ch18} \\\hline
$34.90$ & $39.42$ & $\mathbf{40.62}$ & $40.56$ & $37.85$ & $39.87$
\end{tabular}
\caption{Additive mixture noise using ($\ell^2$-cost/$f^\implicit$). 
Note that GCBD is a blind method, i.e.~the noise distribution is estimated using a GAN with infinite training data.}
\end{subtable}

\begin{subtable}[ht]{\textwidth}
\centering
\begin{tabular}{*{6}{c|} }
noisy &  $\mathcal{D}_{\ell^2}$ & $\mathcal{D}_F$ & $\mathcal{D}_\div$ & BM3D~\cite{DaFo07} & TV~\cite{RuOsFa92} \\\hline
$17.16$ & $28.03$ & $28.09$ & $\mathbf{28.12}$ & $27.05$ & $25.49$
\end{tabular}
\caption{Additive Laplace noise using ($\ell^2$-cost/$f^\implicit$).}
\end{subtable}

\begin{subtable}[ht]{\textwidth}
\centering
\begin{tabular}{*{7}{c|} }
noisy &  $\mathcal{D}_{\ell^2}$ & $\mathcal{D}_F$ & $\mathcal{D}_\div$ & VST+BM3D~\cite{Re18} & CAFC~\cite{Re18} & $\text{MC}^2\text{RNet}_6$-S~\cite{Su19} \\\hline
$9.8$ & $\mathbf{24.44}$ & $\mathbf{24.44}$ & $\mathbf{24.44}$ & $23.54$ & $23.99$ & $24.25$
\end{tabular}
\caption{Multiplicative Poisson noise using ($\ell^2$-cost/$f^\implicit$).}
\end{subtable}

\begin{subtable}[ht]{\textwidth}
\resizebox{\linewidth}{!}{
\begin{tabular}{ *{11}{c|} }
\multirow{2}{*}{noisy} & \multicolumn{2}{|c|}{$\mathcal{D}_{\ell^2}$} & \multicolumn{3}{ |c| }{$\mathcal{D}_F$} &  \multicolumn{3}{ |c|}{$\mathcal{D}_\div$} & \multirow{2}{*}{BM3D~\cite{DaFo07}}& \multirow{2}{*}{TV~\cite{RuOsFa92}} \\\cline{2-9}
& $f^\explicit$ & $f^\implicit$ & $f^\explicit$ & $f^\implicit$ & $f^\Newton$ & $f^\explicit$ & $f^\implicit$ & $f^\Newton$ & & \\ \hline
$8.10$ & $21.14$ & $21.28$ & $24.82$ & $27.23$ & $27.11$ & $28.88$ & $30.97$ & $\mathbf{31.48}$ & $17.92$ & $17.05$
\end{tabular}
}
\caption{Salt-and-pepper noise for various discretization schemes and $\ell^2$-cost.}
\end{subtable}
\caption{PSNR values for distinct denoising tasks with various discretization schemes.}
\label{tab:PSNRDenoisingSV}
\end{table}

\begin{figure}[htb]
\includegraphics[width=\linewidth]{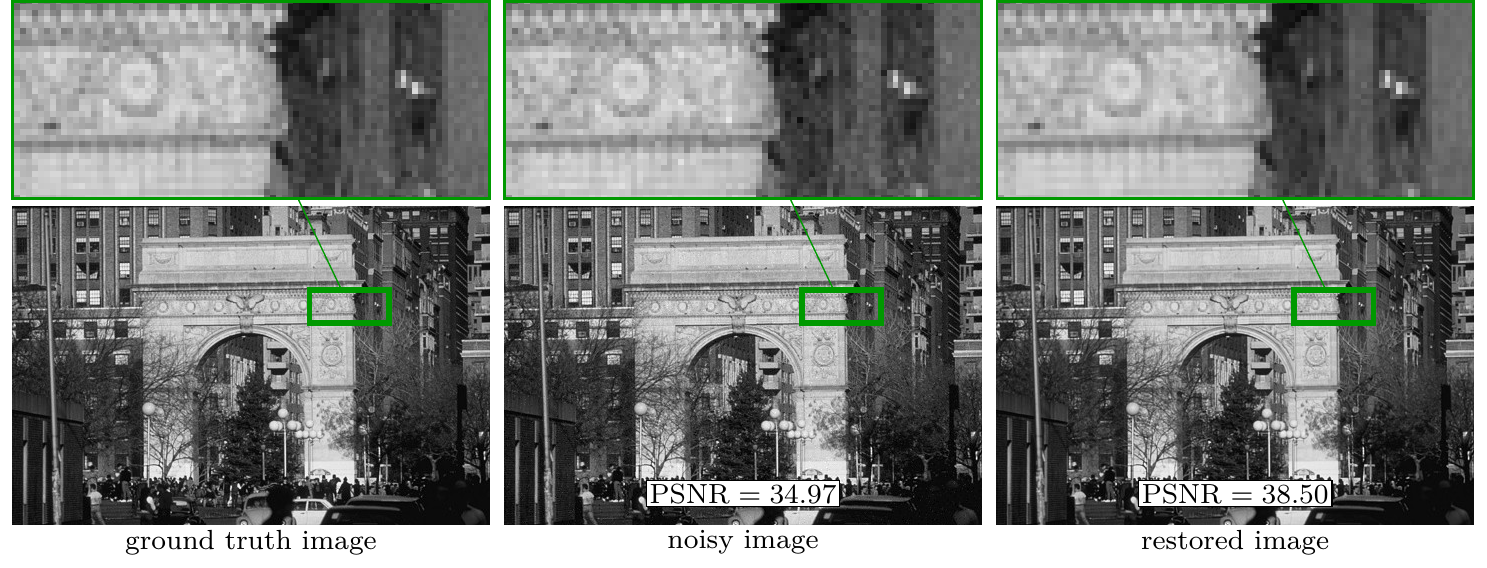}
\caption{From left to right: ground truth image, image corrupted by mixture noise, restored image using ($\ell^2$-cost/$f^\implicit$/$\mathcal{D}_\div$).
The magnification factor of the zoom is~$6$.}
\label{fig:MixedResults}
\end{figure}

\begin{figure}[htb]
\includegraphics[width=\linewidth]{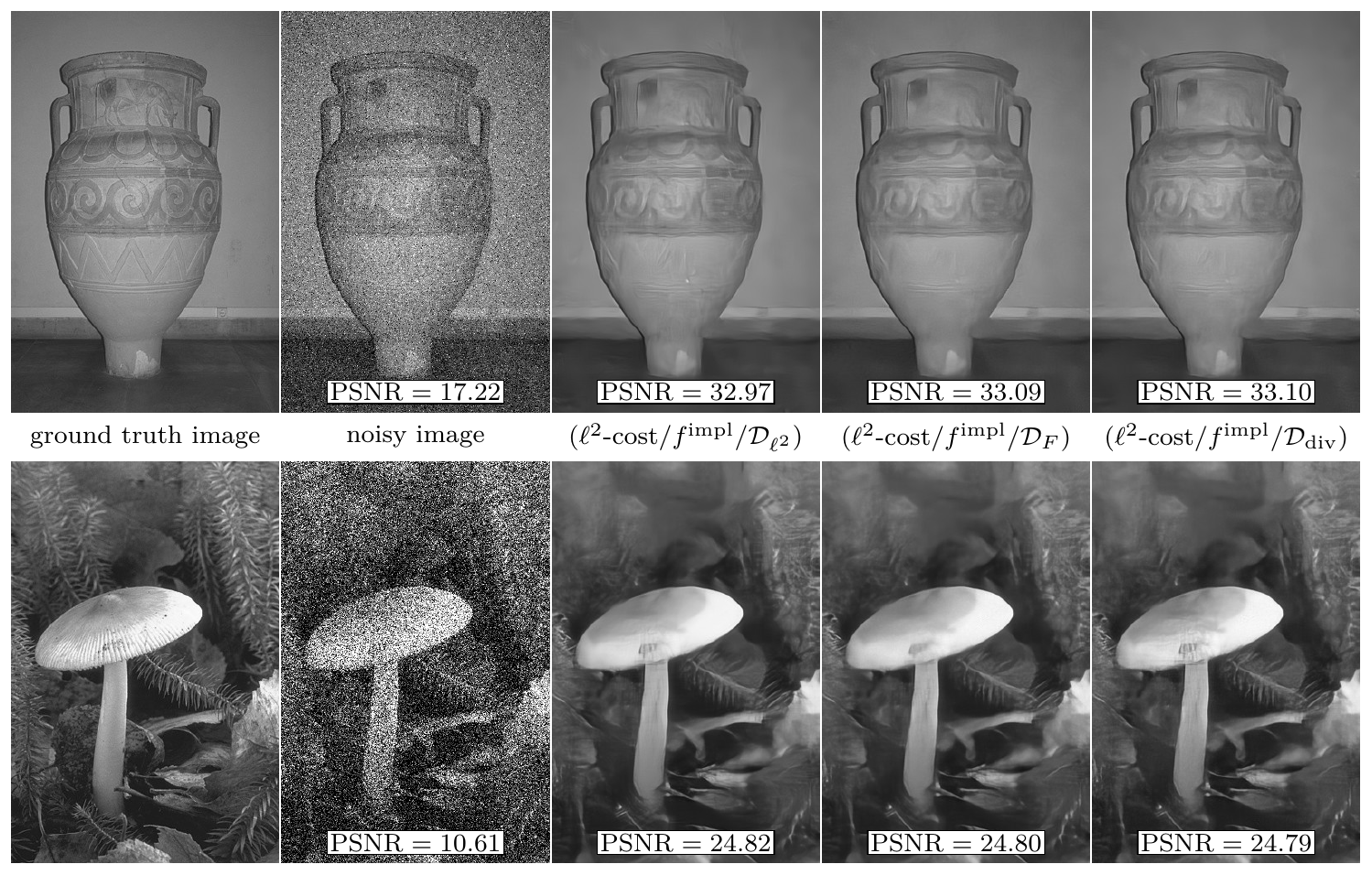}
\caption{From left to right: ground truth images, images corrupted by additive Laplace noise (first row)/multiplicative Poisson noise (second row) along with the restored images for various discretization schemes.}
\label{fig:LaplacePoissonResults}
\end{figure}

\begin{figure}[htb]
\includegraphics[width=\linewidth]{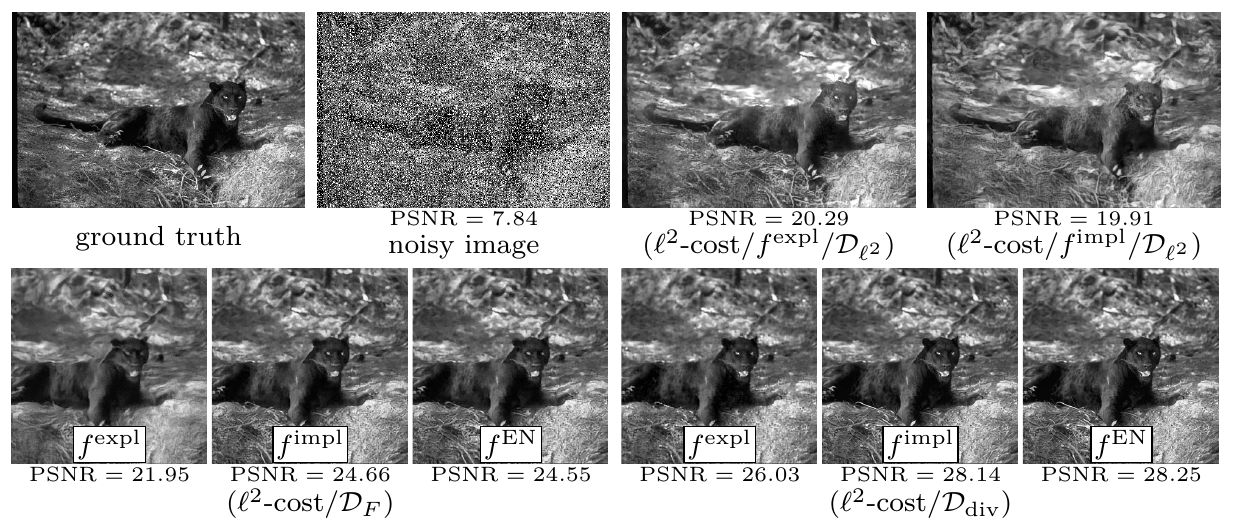}
\caption{Ground truth image, image corrupted by salt-and-pepper noise as well as the restored images using various discretization schemes.}
\label{fig:SNPResults}
\end{figure}

\subsubsection{Demosaicing}
In what follows, we apply energy-based learning to demosaicing, where the linear operator~$A$ is the standard mosaicing operator for the Bayer pattern and~$A_\init$ is the bicubic interpolation operator.

In~\Cref{fig:proxMapsDemosaicing}, the learned data fidelity terms for demosaicing using ($\ell_\iota^1$-cost/$f^\Newton$/$\mathcal{D}_F$) (left) and ($\ell_\iota^1$-cost/$f^\Newton$/$\mathcal{D}_\div$) (right) are depicted.
Note that in all experiments the resulting learned data fidelity terms for Fr\'echet and generalized divergence are approximately identical, respectively.
\Cref{tab:PSNRDemosaicingSV} lists the PSNR values for various discretization schemes as well as competing state-of-the-art methods evaluated on the Panasonic and the Canon data sets~\cite{KhNo14},
where all images have a resolution of $220\times 132$ and $318\times 210$, respectively.
For both data sets, the difference in the PSNR values among different discretization schemes is small and better values are observed when using the $\ell_\iota^1$-cost functional, but in either case we achieve state-of-the-art results.
Note that the best PSNR scores are achieved when using the learned Fr\'echet metric.
\Cref{fig:resultsDemosaicing} depicts the ground truth image, the input image after the bicubic interpolation~$A_\init z$ as well as the restored output images for all considered data fidelity terms using ($\ell_\iota^1$-cost/$f^\Newton$).
As a result, in all cases the quality of the restored images is substantially better than~$A_\init z$--both visibly and in terms of PSNR values.

\begin{figure}[htb]
\includegraphics[width=\linewidth]{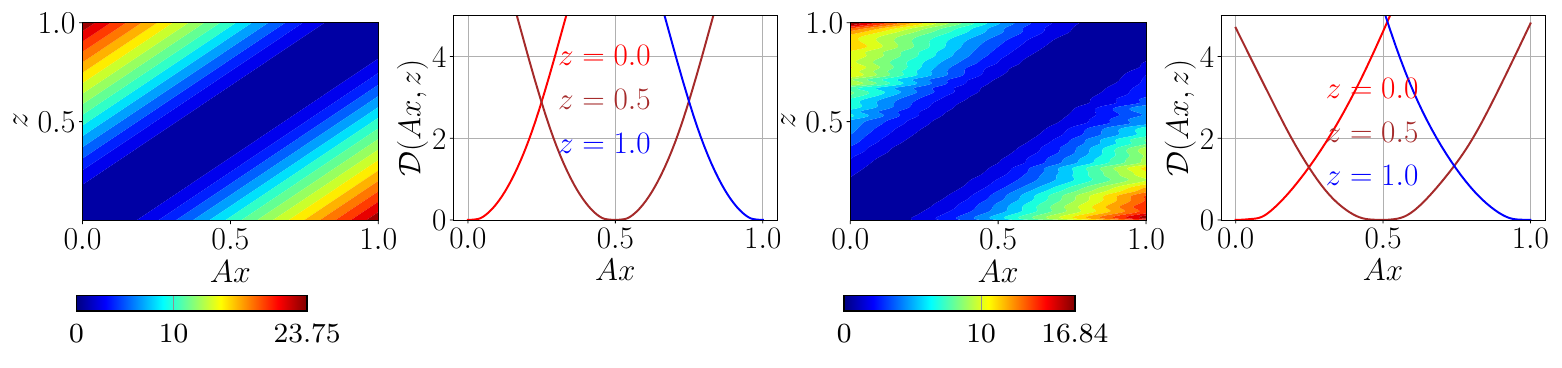}
\caption{
Learned data fidelity terms for demosaicing using ($\ell_\iota^1$-cost/$f^\Newton$/$\mathcal{D}_F$) (left) and ($\ell_\iota^1$-cost/$f^\Newton$/$\mathcal{D}_\div$) (right).
}
\label{fig:proxMapsDemosaicing}
\end{figure}

\begin{table}[ht]
\resizebox{\linewidth}{!}{
\begin{tabular}{c | c || *{8}{c|} }
\multirow{2}{*}{data set}& cost & \multirow{2}{*}{$A_\init z$} &  $\mathcal{D}_{\ell^2}$ & \multicolumn{2}{ |c| }{$\mathcal{D}_F$} &  \multicolumn{2}{ |c|}{$\mathcal{D}_\div$} & \multirow{2}{*}{SEM~\cite{KlHa16}}& \multirow{2}{*}{DeepISP~\cite{ScGi19}} \\\cline{4-8}
&functional& & $f^\explicit$ & $f^\explicit$ & $f^\Newton$ & $f^\explicit$ & $f^\Newton$ & & \\ \hline
\multirow{2}{*}{\textbf{P}} &$\ell_\iota^1$ & \multirow{2}{*}{$30.37$}  & $39.48$ & $39.79$ & $\mathbf{39.88}$ & $39.80$ & $39.80$ & \multirow{2}{*}{38.93} & \multirow{2}{*}{39.31} \\
&$\ell^2$ &  & $39.43$ & $39.45$ & $39.40$ & $39.36$ & $39.37$ & & \\ \hline
\multirow{2}{*}{\textbf{C}} &$\ell_\iota^1$ & \multirow{2}{*}{$32.30$}  & $41.71$ & $\mathbf{42.07}$ & $41.97$ & $41.69$ & $42.02$ & \multirow{2}{*}{41.09} & \multirow{2}{*}{41.70} \\
&$\ell^2$ &  & $41.52$ & $41.60$ & $41.60$ & $41.65$ & $41.65$ & & 
\end{tabular}
}
\caption{PSNR values of demosaicing for various discretization schemes and the Panasonic~(\textbf{P}) and Canon~(\textbf{C}) data set~\cite{KhNo14}.}
\label{tab:PSNRDemosaicingSV}
\end{table}

\begin{figure}[htb]
\includegraphics[width=\linewidth]{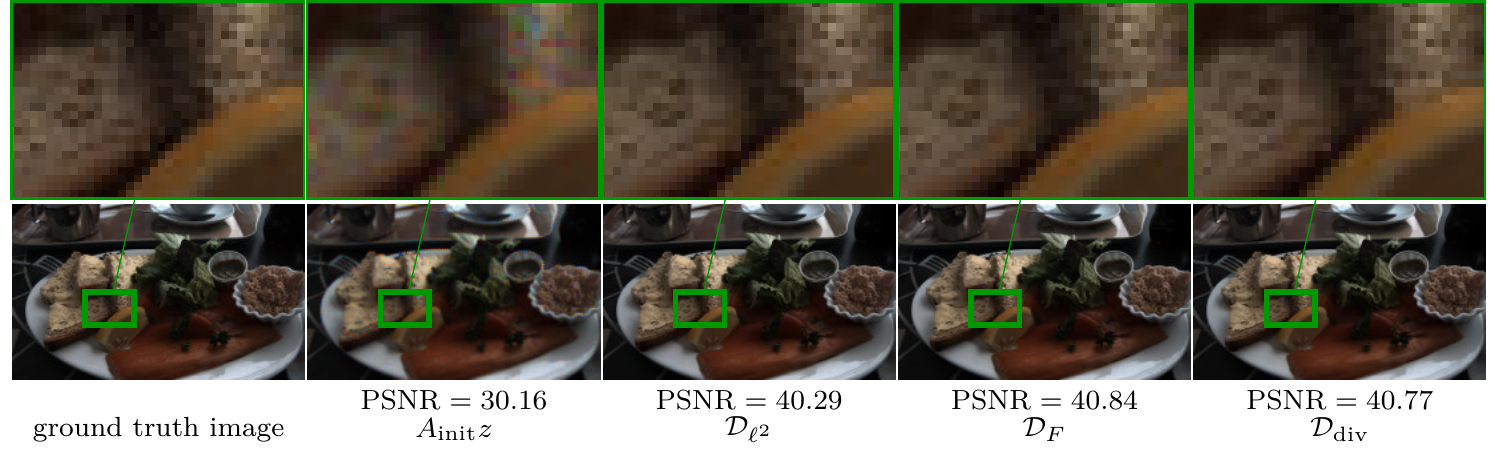}
\caption{From left to right: ground truth image, $A_\init z$, restored output images computed with different learned data fidelity terms using ($\ell_\iota^1$-cost/$f^\Newton$).
The magnification factor of the zoom is~$6$.}
\label{fig:resultsDemosaicing}
\end{figure}

\subsubsection{(Noisy) single image super-resolution}
For single image super-resolution (SISR), the linear operator~$A\in\R^{nC/\omega^2\times nC}$ is given as a Gaussian downsampling operator with scale factor~$\omega\in\N$ and standard deviation $\sigma_A=2$.
Throughout this paper, we set $\omega=3$ and we always use ($\ell_\iota^1$-cost/$f^\Newton$).
Furthermore, we consider both noise-free and noisy SISR, in the latter case additive white Gaussian noise~$\zeta$ with $\sigma_{\zeta}=7.65$ is included. 
To follow evaluation conventions in other super-resolution works (see e.g.~\cite{Zh20}), we also evaluate the PSNR value only on the Y-channel of the YCbCr color space, where we first remove $\omega$ pixels from all sides.

\Cref{fig:proxMapsSISR} visualizes the corresponding learned data fidelity terms in the noise-free (top row) and noisy configuration (bottom row).
In a narrow band around the diagonal, the resulting data fidelity term for the Fr\'echet metric is roughly~$0$, outside of this region steep slopes can be observed.
Interestingly, the level lines of~$\mathcal{D}_\div$ are visually significantly smoother in the noisy case.
\Cref{tab:PSNRSISRSV} lists the resulting PSNR values of our method in comparison with USRNet~\cite{Zh20} as the state-of-the-art benchmark.
We remark that USRNet is trained for multiple kernels, scale factors, and noise levels.
\Cref{fig:resultsSISR} depicts the original high-resolution ground truth image, the low-resolution input and the output images of our method for all three data fidelity terms.
The computed output images are visually and in terms of the PSNR value significantly better than the low-resolution input image.

\begin{figure}[htb]
\includegraphics[width=\linewidth]{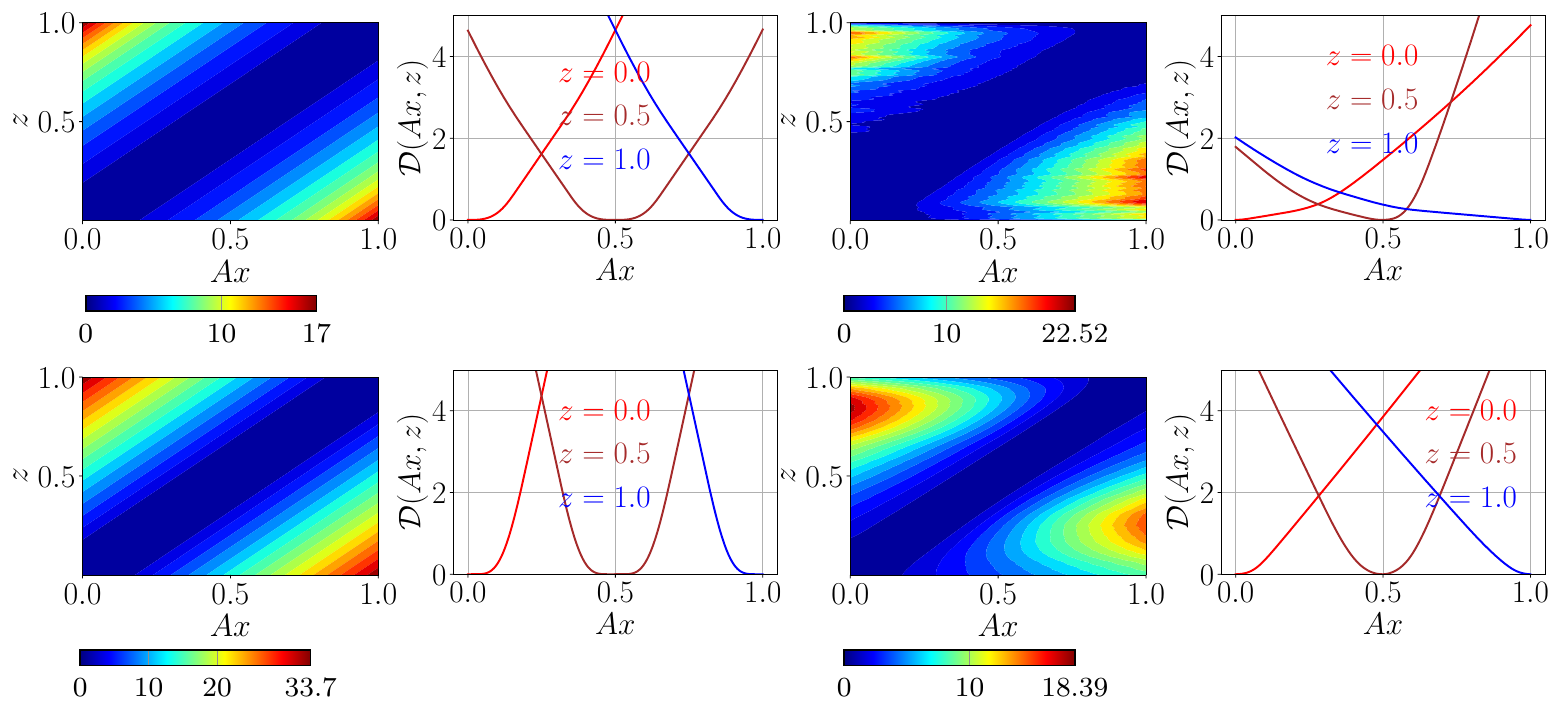}
\caption{
Learned data fidelity terms for SISR along with their respective slice functions using ($\ell_\iota^1$-cost/$f^\Newton$/$\mathcal{D}_F$) (first/second column) and ($\ell_\iota^1$-cost/$f^\Newton$/$\mathcal{D}_\div$) (third/fourth column)
for noise-free (first row) and noisy (second row) configuration.
}
\label{fig:proxMapsSISR}
\end{figure}

\begin{table}[ht]
\centering
\begin{tabular}{c || *{8}{c|} }
& $A_\init\overline{z}$ &  $\mathcal{D}_{\ell^2}$ & {$\mathcal{D}_F$} &  {$\mathcal{D}_\div$} & USRNet~\cite{Zh20} \\\hline
noise-free & $24.67$ & $28.47$ & $28.47$ & $\mathbf{28.49}$ & $27.88$ \\ \hline
noisy & $24.58$  & $26.70$ & $26.71$ & $\mathbf{26.78}$ & $25.57$
\end{tabular}
\caption{PSNR values of non-blind SISR for various discretization schemes.
USRNet is trained on multiple kernels and scale factors simultaneously.}
\label{tab:PSNRSISRSV}
\end{table}

\begin{figure}[htb]
\includegraphics[width=\linewidth]{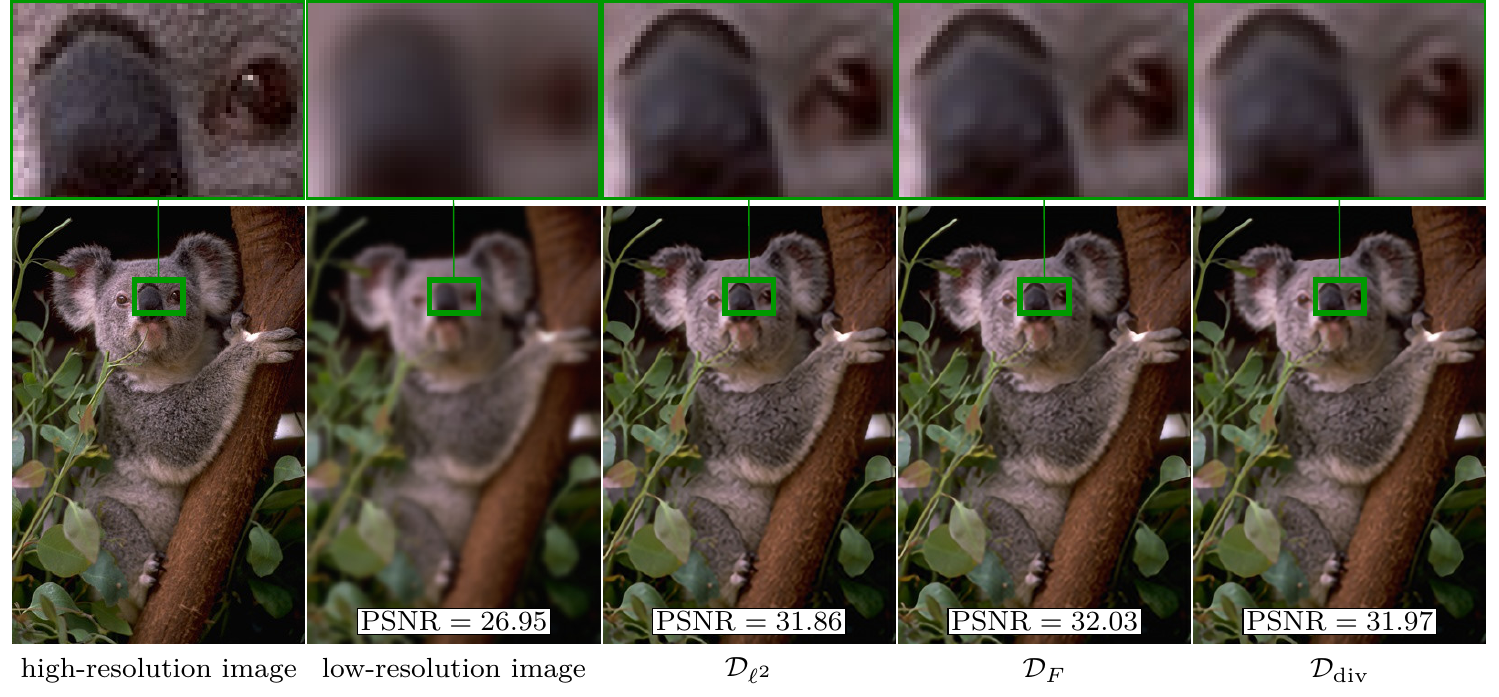}
\caption{From left to right: high-resolution image, noise-free low-resolution image, restored output images computed with different data fidelity terms using ($\ell_\iota^1$-cost/$f^\Newton$).}
\label{fig:resultsSISR}
\end{figure}

\subsection{Shared prior learning}\label{sub:numericsSharedPriorLearning}
Next, we present numerical experiments for shared prior learning.
In this case, we use two separate data sets for the supervised and unsupervised training.
Contrary to energy-based learning, \emph{no ground truth images are available for the unsupervised data set}.
In all experiments, the BSDS400 data set is used for supervised and the DIV2K data set~\cite{AgTi17} for unsupervised training, where both data sets are degraded in different fashions.
The patch size for both the supervised and the unsupervised subproblem is $60\times 60$ due to hardware restrictions, and the batch size is~$10$.
Furthermore, all control parameters are initialized with the corresponding control parameters of a previous supervised training, where we pre-train with $100\,000$ iterations.
For the Wasserstein distance estimation, we use $50$ iterations of \Cref{alg:Wasserstein} and we set $\beta=1$ if not otherwise stated.
In all experiments, the $\ell_\iota^1$-cost functional is used since this choice has empirically proven to yield better performance.
Again, we use the ADAM optimizer with a learning rate of $10^{-4}$, $2000$~iterations, $\beta_1=0.5$ and $\beta_2=0.9$.
If not otherwise specified, we utilize $\F=\DCT$ with $n_\patch=64$ and $N=3610$ (resulting from the number of overlapping patches in $10$~images of size $60\times 60$) in all experiments since this operator has proven to be superior.
We evaluate the performance in terms of the PSNR value of the reconstructions on validation data sets, which are corrupted in the same way as the unsupervised training data set.
Again, we highlight that we are solely aiming at maximizing the performance on the unsupervised subproblem.

\subsubsection{Image denoising}
In the first numerical experiment, we analyze the capability of shared prior learning to adapt to unknown realistic noise distributions.
Therefore, we consider denoising of images corrupted by additive Laplace noise ($\sigma=25$) without the incorporation of ground truth images.
In the supervised subproblem with ($\ell^2$-cost/$f^\implicit$/$\mathcal{D}_{\ell^2}$), the images are degraded by additive Gaussian noise with $\sigma=25$.
We emphasize that we are exclusively interested in the PSNR score for Laplace denoising task and we do not optimize the PSNR value of the Gaussian denoising subproblem.
For the unsupervised subproblem, we utilize ($\ell^2$-cost/$f^\implicit$/$\mathcal{D}_{\div}$).

\begin{table}[ht]
\centering
\begin{tabular}{*{6}{c|} | c }
 noisy &  $\alpha=1$ & $\alpha=0$ &  $\alpha=0.8$ & DnCNN-B~\cite{ZhZu17} & BM3D~\cite{DaFo07} & supervised reference \\\hline
 $24.67$ & $22.48$ & $24.42$ & $\mathbf{27.40}$ & $27.34$ & $27.05$ & $28.12$
\end{tabular}
\caption{PSNR values for unsupervised additive Laplace denoising, where only for the supervised reference ground truth images are available.
For our unsupervised results we used ($\ell^2$-cost/$f^\implicit$/$\mathcal{D}_{\ell^2}$).}
\label{tab:PSNRUSLaplace}
\end{table}

\Cref{tab:PSNRUSLaplace} lists the PSNR values for unsupervised additive Laplace denoising using the $\DCT$ operator and $n_\patch=100$ after $2.000$ iterations.
We highlight that shared prior learning with $\alpha=0.8$ clearly outperforms the competing purely supervised ($\alpha=1$) and unsupervised ($\alpha=0$) methods as well as the state-of-the-art approach.
\Cref{fig:ablationOperator} depicts the dependency of the PSNR value on the balance parameter~$\alpha$ for all operators, where the subscript refers to the varying patch size~$n_\patch$.
To compensate numerical instabilities, we set $\beta=5$ for the patch sizes~$10$ and~$12$.
In the first row, we terminate the optimization at the best PSNR value, for which ground truth images are necessary, whereas in the second row the performance after~$2000$ iterations is shown.
Furthermore, the horizontal black line indicates the PSNR value of the supervised training for additive Laplace noise using the ground truth during training.
We stress that there is hardly any difference in the performance for $\alpha\in[0.1,0.7]$ among all curves when terminating at the optimal iteration (first row),
but a significant increase in the performance is visible when passing from the purely supervised (i.e.~$\alpha=1$) or unsupervised (i.e.~$\alpha=0$) to shared prior learning (i.e.~$\alpha\in(0,1)$).
The $\DCT$ operator performs on average slightly better than the autoencoder, which itself is marginally superior to the identity operator.
As a result, the optimal value of~$\alpha$ strongly depends on the task, the operator and the patch size.
\Cref{fig:LaplaceDenoisingUN} (first row) shows the ground truth image (first image), the noisy input image (second image) corrupted by Laplace noise along with the restored images when trained on a data set deteriorated by Gaussian noise (third image) and Laplace noise (fourth image) in a supervised setting.
In the second row, several output images for shared prior learning are presented for various choices of~$\alpha$ and~$\F$, where the patch size is~$n_\patch=10$ in all experiments.
Since the Laplace distribution is heavier-tailed than the Gaussian distribution, isolated outliers result in deteriorated pixels in the restored images--apart from the supervised restoration trained on Laplace noise.

\begin{figure}[htb]
\includegraphics[width=\linewidth]{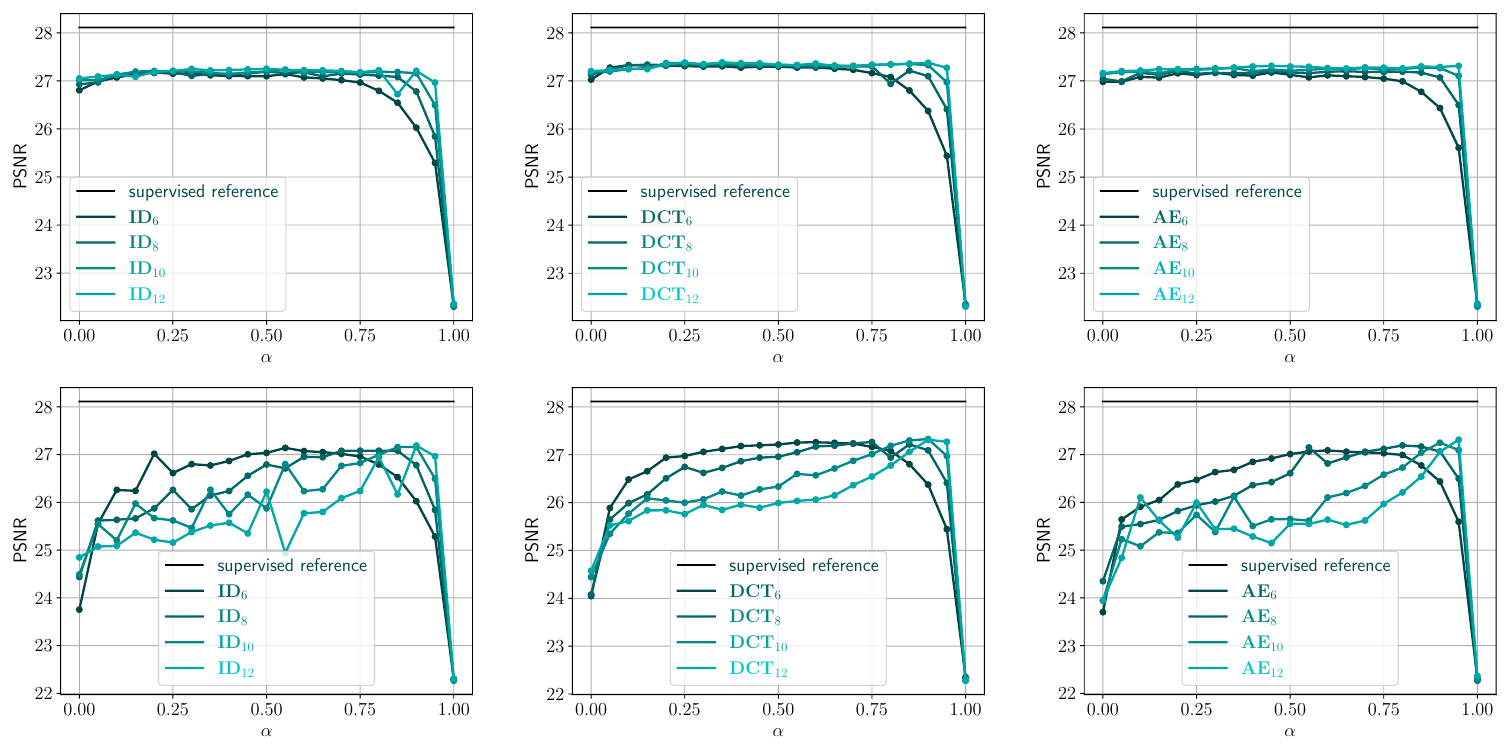}
\caption{Ablation study of the PSNR value for blind additive Laplace denoising using different operators and varying~$n_\patch$.
The optimization algorithm is stopped at the optimal iteration (first row) and after $2000$~iterations (second row).}
\label{fig:ablationOperator}
\end{figure}

\begin{figure}[htb]
\includegraphics[width=\linewidth]{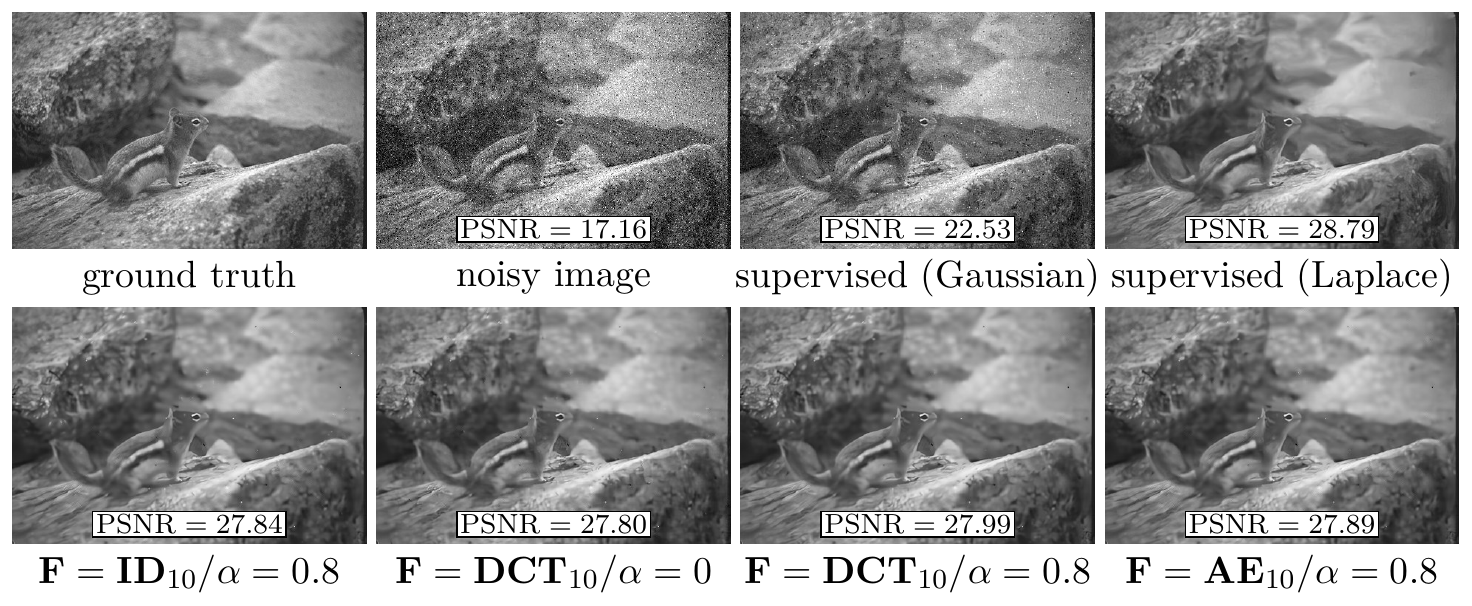}
\caption{First row (from left to right): ground truth image, noisy image, output image (supervised reconstruction) trained on images corrupted by Gaussian and Laplace noise.
Second row: prototypic images of the aforementioned ablation study in \Cref{fig:ablationOperator}.}
\label{fig:LaplaceDenoisingUN}
\end{figure}

\subsubsection{Single image super-resolution}
In what follows, we apply shared prior learning to SISR.
Throughout all experiments, the linear operator~$A$ in the supervised reference task is a Gaussian downsampling kernel with $\sigma_A=2$ and a scale factor of~$\omega=3$.
Furthermore, no noise is added and we use ($\ell_\iota^1$-cost/$f^\Newton$/$\mathcal{D}_F$).
For the unsupervised task using ($\ell_\iota^1$-cost/$f^\Newton$/$\mathcal{D}_F$), we assume that the linear operator~$\widetilde{A}$ is unknown during training and additive noise~$\zeta$ is included in some cases.
In detail, we consider three scenarios:
\begin{enumerate}[label=(SISR \arabic*)]
\item\label{item:SISR1}
$\widetilde{A}$ coincides with~$A$ and Gaussian noise~$\zeta$ with $\sigma_{\zeta}\in\{2.55,7.65\}$ is added.
\item\label{item:SISR2}
$\widetilde{A}$ is a Gaussian downsampling kernel with $\sigma_{\widetilde{A}}=1.2$ and no noise is added.
\item\label{item:SISR3}
$\widetilde{A}$ is a Gaussian downsampling kernel with $\sigma_{\widetilde{A}}=1.2$ and Gaussian noise $\zeta$ with $\sigma_{\zeta}\in\{2.55,7.65\}$ is added.
\end{enumerate}
We compare our results with the corresponding supervised outcome and two recently proposed approaches for non-blind SISR.
Contrary to the latter approach, we do not perform any kernel estimation here.
All experiments are conducted for color images, where only the $Y$-channel of~$x(1,\widehat z)$ of the YCbCr color space is used in all further computations.
We use $\alpha=0.03$ since the typical empirical Wasserstein distance is $1$ and the $\ell_\iota^1$-norm on the supervised task averages around $300$ for our chosen parameters.
In all experiments, we use $2000$~iterations for the optimizer.
\Cref{tab:PSNRSISRUnSISR} lists the PSNR results for all three scenarios.
We stress that shared prior learning with~$\alpha=0.03$ yields the best performance in all cases apart from the supervised case, in which ground truth images are available for training.
\Cref{fig:SISRUNAll} depicts resulting image sequences for all three scenarios with different ground truth images (first column), where $\sigma_{\zeta}=2.55$ in~\ref{item:SISR1} and~\ref{item:SISR3} is used.
As an initialization, we use $A_\init z$, where $A_\init$ is the adjoint operator of the known operator~$A$ multiplied by $\omega^2$.
The resulting restorations are depicted in the third column for $\alpha=1$ and in the fourth column for $\alpha=0.03$.
In the last column, the corresponding reference images for supervised training with known ground truth are shown.
We highlight that for some image sequences the restored images generated by shared prior learning with $\alpha=0.03$ outperform the associated supervised results of the last column (which is not valid for the entire validation data set).

\begin{table}[htb]
\centering
\begin{tabular}{c | c ||*{4}{c|} c | | c }
& &   & shared prior & shared prior & & & \\
& & & learning & learning & USRNet & IRCNN & supervised\\
& $\sigma_{\zeta}$ & $A_\init z$ & ($\alpha=1$) & ($\alpha=0.03$) & ~\cite{Zh20} & ~\cite{Zh18} & reference \\\hline\hline
\multirow{2}{*}{\ref{item:SISR1}} & $2.55$ & $24.66$ & $24.81$ & $\mathbf{27.18}$ & $26.78$ & $25.33$ & $27.75$ \\
& $7.65$ & $24.58$ & $17.81$ & $\mathbf{26.08}$  & $25.57$ & $24.39$ & $26.78$ \\\hline
\ref{item:SISR2} & $0$ & $25.20$ & $20.68$ & $\mathbf{28.17}$  & $27.76$ & $26.89$ & $28.50$ \\\hline
\multirow{2}{*}{\ref{item:SISR3}} & $2.55$ & $25.29$ & $20.10$ & $\mathbf{27.85}$  & $27.40$ & $26.13$ & $28.40$ \\
 & $7.65$ & $25.30$ & $19.22$ & $\mathbf{26.87}$  & $26.52$ & $24.68$ & $27.64$
\end{tabular}
\caption{PSNR values for all three scenarios considered.}
\label{tab:PSNRSISRUnSISR}
\end{table}

\begin{figure}[htb]
\includegraphics[width=\linewidth]{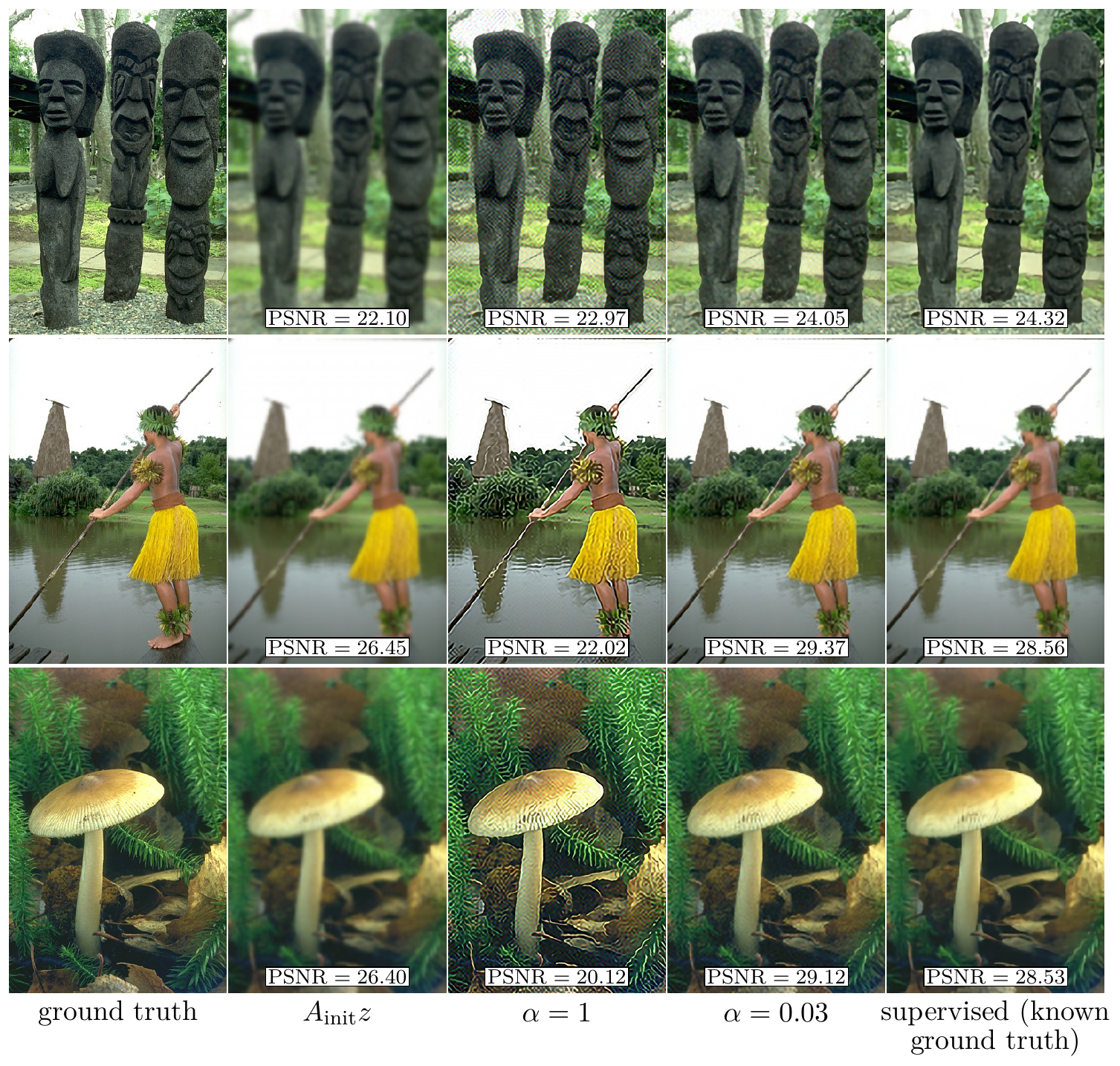}
\caption{Results for tasks \ref{item:SISR1} ($\sigma_{\zeta}=2.55$, first row), \ref{item:SISR2} (second row) and \ref{item:SISR3} ($\sigma_{\zeta}=2.55$, third row) with patch size $n_\patch=6$.}
\label{fig:SISRUNAll}
\end{figure}

\section{Conclusion}
In this paper, we have introduced a shared prior learning approach for learning variational models in imaging.
We have demonstrated the superiority of the proposed method for various applications--even in the absence of ground truth data.
In particular, the only task-specific hyperparameter is the learning rate, that is why our method can easily be adapted to a variety of imaging problems.
The learned data fidelity terms and proximal operators strongly depend on the task and the discretization scheme.
The consistency of the discretization schemes for increasing iteration steps~$S$ in terms of Mosco convergence has been verified.
The stability of the method including robustness against adversarial attacks and empirical upper bounds of the generalization error 
have not been addressed here, but the overall approach would be analogous to~\cite{KoEf20a}.
Finally, we highlight that our method is also applicable to nonlinear inverse problems, which will be analyzed in future work.

\section*{Acknowledgements}
All authors acknowledge support from the European Research Council under the Horizon 2020 program, ERC starting grant HOMOVIS (No. 640156).
We thank Vanessa Effland, Joana Grah and Andreja Ojster\v{s}ek for fruitful discussions and support.

\bibliography{main}
\bibliographystyle{alpha}

\end{document}